\documentclass{article}

\usepackage{arxiv}

\usepackage[utf8]{inputenc} 
\usepackage[T1]{fontenc}    
\usepackage{hyperref}       
\usepackage{url}            
\usepackage{booktabs}       
\usepackage{amsfonts}       
\usepackage{nicefrac}       
\usepackage{microtype}      
\usepackage{graphicx}

\usepackage{algorithm}
\usepackage{algorithmicx}
\usepackage{algpseudocode}
\usepackage{multirow}
\usepackage{amsfonts}
\usepackage{amsmath}
\usepackage{amsthm}
\usepackage{subcaption}

\usepackage{soul}

\usepackage[acronym]{glossaries}
\usepackage[table,x11names]{xcolor}

\newtheorem{theorem}{Theorem}
\newacronym{m-c}{M-C}{Multilabel Classifier or classification} 
\newacronym{b-c}{B-C}{Binary Classifier or classification}

\title{Meta learning with language models: Challenges and opportunities in the classification of imbalanced text}

\author{
 Apostol Vassilev\\
  National Institute of Standards and Technology\\
  Gaithersburg, USA \\
  \texttt{apostol.vassilev@nist.gov} \\
   \And
 Honglan Jin\\
  National Institute of Standards and Technology\\
  Gaithersburg, USA\\
  \texttt{honglan.jin@nist.gov} \\
  \And
  Munawar Hasan\\
  National Institute of Standards and Technology\\
  Gaithersburg, USA\\
  \texttt{munawar.hasan@nist.gov} \\
}

\begin{document}
\maketitle

\begin{abstract}
Detecting out of policy speech (OOPS) content is important but difficult. While machine learning is a powerful tool to 
tackle this challenging task, it is hard to break the performance ceiling due to factors like 
quantity and quality limitations on training data and inconsistencies in
OOPS definition and data labeling~\cite{fortuna2018survey}.
To realize the full potential of available limited resources, we propose a meta learning technique (MLT)
that combines individual models built with different text representations. 
We analytically show that the resulting technique is numerically stable and produces 
reasonable combining weights. We combine the MLT with a threshold-moving (TM) technique 
to further improve the performance 
of the combined predictor on highly-imbalanced in-distribution and out-of-distribution datasets. We also provide computational results to show the statistically 
significant advantages of the proposed MLT approach.\footnote{DISCLAIMER: This paper is not subject to copyright in the United States. Commercial products are identified in order to adequately specify certain procedures. In no case does such identification imply recommendation or endorsement by the National Institute of Standards and Technology, nor does it imply that the identified products are necessarily the best available for the purpose.}\\[0.1cm]

All authors contributed equally  to this work.

\end{abstract}


\section{Introduction}
\label{sec:intro}
Out of policy speech (OOPS) has permeated social media with serious consequences for both individuals and society. Although it comprises a small fraction of the content generated daily on social media, sifting through the data to quickly identify and eliminate the toxic content is difficult. The scale of this problem has long passed a threshold that requires automated detection. Yet it remains to be a challenging problem for machine learning (ML) due to the way OOPS manifests itself in datasets: context-dependent, nuanced, non-colloquial language that may even be syntactically incorrect. Because the OOPS content of the dataset is usually only a small fraction of the overall size, there is a high imbalance between OOPS and in-policy text. Related to this, there are not many high-quality labeled datasets with consistent definitions of OOPS and in-policy content. The difficulties are exacerbated further by significant differences in the distributions of the datasets that the model has been trained on and the data it sees during deployment. When faced with all of these challenges, ML models applied to natural language processing (NLP) tasks quickly reach a performance ceiling that limits their usefulness for sensitive tasks, such as OOPS detection.     

A recent comprehensive evaluation of the performance of NLP models on many tasks~\cite{liang2022holistic} has revealed a relationship between the accuracy and the access to the model, showing a consistent gap between the higher accuracy of non-open models and that of open models. This finding suggests troubling power dynamics related to language models.  Attempts to close that gap by improving the weaker language models by finetuning them on outputs from a stronger proprietary model have not been successful \cite{gudibande2023false}. 

We approach this problem differently with a meta-learning technique for combining multiple open NLP models with the goals of breaking through the ceiling of each individual model and attaining higher combined performance on a large number of OOPS detection tasks, at a par with the performance of the best non-open model on these tasks. We make three main contributions: \textbf{i) a meta learning technique (MLT)} with language models with statistically significant performance improvements on high-quality recent benchmark datasets specifically curated for many OOPS detection tasks, \textbf{ii) theoretical analysis of the MLT} that establishes it as computationally stable and results in reasonably valued combiner weights for the participating language models, backed by numerical results in full agreement with the theory, 
and \textbf{iii) combining MLT with threshold-moving (TM)}, further improves the performance on highly-imbalanced in-distribution and out-of-distribution datasets, thus closing the performance gap with the best non-open model. Our results also suggest that smaller models can outperform much larger models on tasks of interest. 

This paper is organized as follows. We review related work to establish a baseline for our research. Then we introduce our approach based on MLT with language models and TM. Next, we provide extensive numerical results that demonstrate the advantages of our approach. We end with conclusions.    		
\section{Related Work}
\label{sec:related}
The inception of attention mechanisms~\cite{graves2014neural, bahdanau2014neural, luong2015effective} stimulated the development of many  modeling techniques for NLP problems. The transformer network, first introduced by ~\cite{vaswani2017attention}, has since become the dominant technology used by the top performing modern NLP models, including large language models~\cite{devlin2018bert, yang2019xlnet, brown2020language, rae2021scaling}. Despite the enormous progress~\cite{minae2022}, NLP remains a hard target for ML, and the technology is still unable to fully capture text semantics in many specific contexts. 

The task of OOPS detection in particular adds complexities like ethnic slurs, incorrect or shortened words, ambiguous grammar, sarcasm, context-dependent belligerence, and a lack of consensus about the definition of OOPS~\cite{vidgen2019challenges, waseem2017understanding, rossini2020beyond, fortuna2018survey}. 
Moreover, it is difficult to differentiate OOPS from offensive language~\cite{davidson2017automated}.
~\cite{poletto2021resources} points out major issues like strong subjective
interpretations that results in  different subjective perceptions in a variety
of systems, each trained on a different resource.
~\cite{yin2021towards} discusses the low generalizability of OOPS models.
These difficulties are further exacerbated by the low-quality datasets available to researchers for OOPS detection. For example, \cite{davidson2019racial} found that racial bias is present in most of the available datasets~\cite{waseem2016you, davidson2017automated, golbeck2017large, founta2018large}. Lack of agreement among the annotators and human errors add more complexity to already low-quality and highly imbalanced data. The absence of context awareness is yet another notable pitfall in such datasets.  All of these factors have hindered the research into OOPS detection.

In an attempt to rectify some of these problems, ~\cite{vidgen2021introducing} have created a high-quality dataset specifically curated for multiple OOPS detection tasks. While still highly-imbalanced, this dataset uses distinct primary categories, context-aware annotation, and high-quality regressive annotation. To enable more targeted assess of various model functionalities with a large suite of functional tests, \cite{rottger2020hatecheck} have developed a dataset and used it to asses the performance of open-source language models and some commercial models, revealing critical weaknesses in all of them. There are many other commercial models in use by the industry (e.g.,   Meta's combination of linformer~\cite{wang2020linformer} and reinforcement integrity optimizer, Google Jigsaw’s Perspective Application Programming Interface (API)~\cite{GooglePerspective},  Two Hat’s SiftNinja~\cite{TwoHat}), but their performance has not been independently verified. The work of \cite{rottger2020hatecheck} suggests that some skepticism about such unverified performance claims is warranted.  

The idea of combining models to gain better performance has a long history in machine learning, starting with the linear stacking regression approach of ~\cite{breiman1996stacked} for tabular data. Later \cite{kittler1998combining} extended the idea and presented a systematic approach to combining using Bayesian rules. The recent advancements in deep learning and available modern computational power have led to ensemble techniques~\cite{opitz1999popular}, which deliver good performance in computer vision tasks~\cite{kumar2016ensemble, paul2018predicting, perez2019solo, savelli2020multi}. In contrast, the idea of ensemble techniques in NLP tasks has hardly been explored; see \cite{luong2015effective},~\cite{devlin2018bert} and~\cite{cer2018universal}. One of the reasons for this is that the individual NLP models have distinct encoding, which leads to the creation of fundamentally distinct hyperspaces that in turn may result in disappointing performance. Recently, a promising technique for combining language models has been proposed and applied to sentiment analysis - a binary classification task, on a well-balanced dataset~\cite{ssnet}.

High data imbalance and transfer learning, when the training dataset and the data with which the model is used in deployment have different distributions, pose serious challenges to machine learning. If left unchecked, these challenges can lead to a severe performance degradation with dire consequences in critical applications, such as OOPS detection.  Many techniques for dealing with these challenges have emerged in the literature: oversampling and undersampling~\cite{kotsiantis2006handling}, the weighted class approach~\cite{vidgen2021introducing,rottger2020hatecheck,byrd2019effect}, and threshold-moving (TM)~\cite{brownlee2020gentle, collell2018simple, gharroudi2015ensemble, fan2007study}. Oversampling and undersampling require modifications to the original datasets~\cite{kotsiantis2006handling}, which is either undesirable or impossible in many applications. The analysis in~\cite{byrd2019effect} shows that the effect of class weighting tends to be prominent initially but vanishes as the model training progresses, effectively making the model more skewed towards the majority class. Moreover, the weighted class method presents many challenges computationally. It assigns the class weights before model training  using a class ratio distribution. This heuristic assignment is not guaranteed to produce an optimal result. To search for optimal class weights, the models must be retrained every time a class weight is changed. Repeated model training requires not only significant computing resources but also a long time. TM,  on the other hand, is a method that maximizes a selected performance measure without retraining. Additionally, TM can be trained on a different performance measure of interest easily, as it is a plug-in method that works during the post-training stage. Due to all of these factors, we utilize TM to address the class imbalance problem.  	
\section{Methodology}
\label{sec:methodology}
Imbalanced data distribution is prevalent in many
application domains, including OOPS detection. Metrics like the F1 score are a better measure than accuracy for uneven data distribution~\cite{sun2009classification}.
We choose macro-averaged F1 scores to report our performance because we treat
all classes equally. Additionally, We report accuracy when applicable, as an absolute 
improvement should ideally be made on both F1 and accuracy, regardless of data distribution. For brevity, in this paper, we will refer to macro-averaged F1 as "\textit{macroF1}".

In multi-label classification, the F1 score for class c is calculated as follows:

$$F_c = \frac{2 * (precision_c * recall_c)}  {(precision_c + recall_c)} = \frac{2TP_c} {(2TP_c + FP_c + FN_c)}$$

where $TP_c$, $FPl_c$, $FN_c$ represents the counts of true positives, false positives, and false negatives for class c, respectively.

Then, the macro-averaged F1 score is obtained by calculating the average of the F1 scores of each class.

Throughout this paper, we will denote a multi-label classifier or classification as~\acrshort{m-c}, and 
a binary classifier or classification as~\acrshort{b-c}.

First, we introduce an MLT architecture to combine predictions from
multiple models. We assume that the number of texts in a corpus $\mathbb{D}$ is large. Let there be $n_c$ classes of texts in $\mathbb{D}$, labeled $1, …, n_c$, for some integer $n_c > 2$. Here, we define a predictor $\text{F}$ in terms of a neural network with dense layers and a sigmoid. The neural network computes the weights $\{w_i\}_{i=1}^K$ for combining the participating models $M_1,\;  ...,\; M_K$. Let $\mathbf{y}_i(\tau^j)$ be the probability estimate computed by the $i$-th model on the $j$-th text in $\mathbb{D}$.

We define the MLT predictor  as a combination of $K>1$ individual predictors:
\begin{align} 
	\label{eq:combinedpredictor}
  	\mathbf{y}(\tau) &= \sum_{i=1}^K w_i\mathbf{y}_i(\tau),\; \forall\tau\in\mathbb{D},\\
 	\text{where either}\; w_i &\geq 0,\; \forall i;\;\;\; \text{or} \;w_i \leq 0,\; \forall i.&\label{eq:nonnegative}
\end{align}
 The real-valued functions $\mathbf{y}_i$  with a range of interval $[0,\; 1]$ have values that correspond to the probability of being assigned to a class $I^{(C)}$ and are assigned to it with respect to the class threshold test with some value $t\in (0,\; 1)$:

\begin{equation}
  \tau \in
  \begin{cases} \mathbf{I}^{(C)} & \text{, if $\sigma(\mathbf{y}(\tau)) \geq t$,}\\
    \mathbf{I}^{(\bar{C})} & \text{, otherwise.}
  \end{cases}
  \label{eq:threshold}
\end{equation}
Here $\sigma(x)$ is the sigmoid function.
Because $\mathbf{y}_i(\tau)$ is, with ranges, shifted with respect to the domain of the sigmoid, to get accurate classification one needs to shift the range of $\mathbf{y}$,
\begin{equation}
  \sigma_b(\mathbf{y}(\tau)) = \sigma(\mathbf{y}(\tau) + b),
  \label{eq:biasedsigma}
\end{equation}
for some $b>0$ if $w_i\leq 0, \forall i$ or $b<0$ if $w_i\geq 0, \forall i$.  

Secondly, we employ the threshold-moving (TM) technique to post-process predictions from MLT, aiming to alleviate data imbalance issues \cite{kotsiantis2006handling, collell2018simple}. We refer to this combination as MLT-plus-TM.
TM involves utilizing a trained threshold value to map probabilities to class labels, thereby optimizing selected performance metrics such as the \textit{macroF1} score of the model.

Subsequently, we construct classification systems, denoted as MLT-plus-TM, for both multi-label classifiers (~\acrshort{m-c}) and binary classifiers (~\acrshort{b-c}). Both systems utilize Gold labeled data from Contextual Abuse Dataset (CAD)~\cite{vidgen2021introducing}.


For both classification tasks, we begin by constructing five individual models using five
different embeddings. Subsequently, we combine the results using the approach outlined in \eqref{eq:combinedpredictor}-\eqref{eq:biasedsigma}, followed by applying the threshold-moving (TM) technique in both~\acrshort{m-c} and~\acrshort{b-c}.

To establish a notion about the reasonableness of the combiner weight values in \eqref{eq:combinedpredictor}, we show that the following theorem holds.
\begin{theorem}\label{thm1}
Let $W=\sum_{i=1}^K w_i$ with $w_i$ computed by the procedure \eqref{eq:combinedpredictor}-\eqref{eq:biasedsigma}. Then for any class $C\in\mathbb{D}$, $W$ is either close to 1 or -1. 
\end{theorem}
\begin{proof}\footnote{ A detailed proof is provided in Appendix~\ref{appendix}.}

For a real-valued function  $\mathbf{y}_i(\tau)$ let us define
\begin{equation}
  ||\mathbf{y}(\tau)||_t^2 =  \sum_{\tau \in \mathbb{D}} \mathbf{I}^{(C)}_t(\mathbf{y}(\tau))^2,
    \label{eq:classl2norm}
\end{equation} where $\mathbf{I}^{(C)}_t$ is the assigned class for $\mathbf{y}(\tau)$ with respect to the threshold $t$ according to \eqref{eq:threshold}. 

 Let $\mathbf{u}$ be defined on the texts in the corpus $\mathbb{D}$, so that $ \mathbf{I^{(C)}_t}(\mathbf{u}(\tau))$ is equal to the label assigned to each text $\tau$ in $\mathbb{D}$ that also belong to class $C$. In other words, $\mathbf{u}$ is the true label on each class.
Similarly, we define
\begin{equation}
  ||\mathbf{y}(\tau) - \mathbf{z}(\tau)||_t^2 =  \sum_{\tau \in \mathbb{D}} (\mathbf{I^{(C)}_t}(\mathbf{y}(\tau)) - \mathbf{I^{(C)}_t}(\mathbf{z}(\tau)))^2.
\label{eq:classdiffl2norm}
\end{equation}
We assume that each predictor $\mathbf{y}_i$ is decent, i.e., $|| \mathbf{u} -\mathbf{y}_i||_t$ is small, compared to $||\mathbf{u}||_t$. Let
$$
  \hat{\mathbf{y}}(\tau) = \frac{1}{|W|}\mathbf{y}(\tau)= \sum_{i=1}^K \frac{w_i}{|W|}\mathbf{y}_i(\tau)
  $$
be the interpolation predictor constructed as a linear combination of $\mathbf{y}_i$ with coefficients that sum up to $\pm 1$. If the individual predictors are good then the interpolation predictor $\hat{\mathbf{y}}$ is also good, i.e., $||\mathbf{u} - \sigma_b(\hat{\mathbf{y}})||_t$ is small relative to $||\mathbf{u}||_t$.

Let $\mathbb{L}_t(x)$ be a linear approximation of $\sigma(x)$ for some constant $t>0$, such that  $\mathbb{L}_t(x)$ minimizes $||\mathbb{L}_t(x) - \sigma(x)||_t$. Note that any straight line passing through the point $(\ln(\frac{t}{1-t}),\, t)$ and having the same slope as $\sigma^\prime(\ln(\frac{t}{1-t}))$ satisfies  $||\mathbb{L}_t(x) - \sigma(x)||_t=0$. Take
$$||\mathbf{u}-\sigma_b(\mathbf{y})||_t = ||\mathbf{u}-W\mathbf{u} + W\mathbf{u} - \sigma_b(\mathbf{y})||_t$$ and consider the case $b< 0$ and $w_i\geq 0$ for $\forall i$ in \eqref{eq:biasedsigma}. Considering the two cases for $W$ ($W\leq 1$ and $W>1$), applying the triangle inequality and using the interpolation predictor $\hat{\mathbf{y}}(\tau)$ and the linear approximation $\mathbb{L}_t(x)$, one gets
\begin{equation}
\label{eq:leftestimateLocal}
  W \geq\frac{ ||\mathbf{u}||_t - ||\mathbf{u}-\sigma_b(\mathbf{y})||_t}{||\mathbf{u}||_t + ||\mathbf{u} - \sigma_b(\hat{\mathbf{y}})||_t }.
\end{equation} and 
\begin{equation}
  W \leq \frac{ ||\mathbf{u}||_t + ||\mathbf{u}-\sigma_b(\mathbf{y})||_t}{||\mathbf{u}||_t -  ||\mathbf{u} - \sigma_b(\hat{\mathbf{y}})||_t}.
\label{eq:rightestimateLocal}
\end{equation}
Next, consider the case $b> 0$ and $w_i\leq 0$ for  $\forall i$ in \eqref{eq:biasedsigma}. As before, we consider the two cases for $w$ ($1 + W>0$ and $W <-1$). Again, with the help of the triangle inequality, the interpolation predictor $\hat{\mathbf{y}}(\tau)$, and the linear approximation  $\mathbb{L}_t(x)$, one gets
\begin{equation}
W\leq -\frac{||\mathbf{u}||_t - ||\mathbf{u}-\sigma_b(\mathbf{y})||_t}{||\mathbf{u}||_t+||\mathbf{u} - \sigma_{b}(\mathbf{\hat y})||_t}
\label{eq:leftnegativeWLocal}
\end{equation}
and
\begin{equation}
W\geq -\frac{||\mathbf{u}||_t + ||\mathbf{u}-\sigma_b(\mathbf{y})||_t}{||\mathbf{u}||_t-||\mathbf{u} - \sigma_{b}(\mathbf{\hat y})||_t}
\label{eq:righttnegativeWLocal}
\end{equation}
Combining  \eqref{eq:leftestimateLocal}, \eqref{eq:rightestimateLocal}, \eqref{eq:leftnegativeWLocal}, and \eqref{eq:righttnegativeWLocal} completes the proof. 
\end{proof}

  The result of Theorem~\ref{thm1} shows that MLT tends to balance the contributions of the individual models to the joint prediction, 
  which is important for obtaining better combined performance than that of the individual models. We illustrate this with computational examples in Section~\ref{sec:results}. 


\begin{figure}[!t]
    \centering
    \caption{\textbf{Dataset for M-C: } In the figure below, we see the CAD {\fontfamily{qcr}\selectfont train}, {\fontfamily{qcr}\selectfont dev}  and {\fontfamily{qcr}\selectfont test} splits with number of samples for each label. It is clear from the figure that the majority of the samples are neutral.}
    \label{fig:cad_mc_dataset}
    {
    \includegraphics[width=0.85\textwidth]{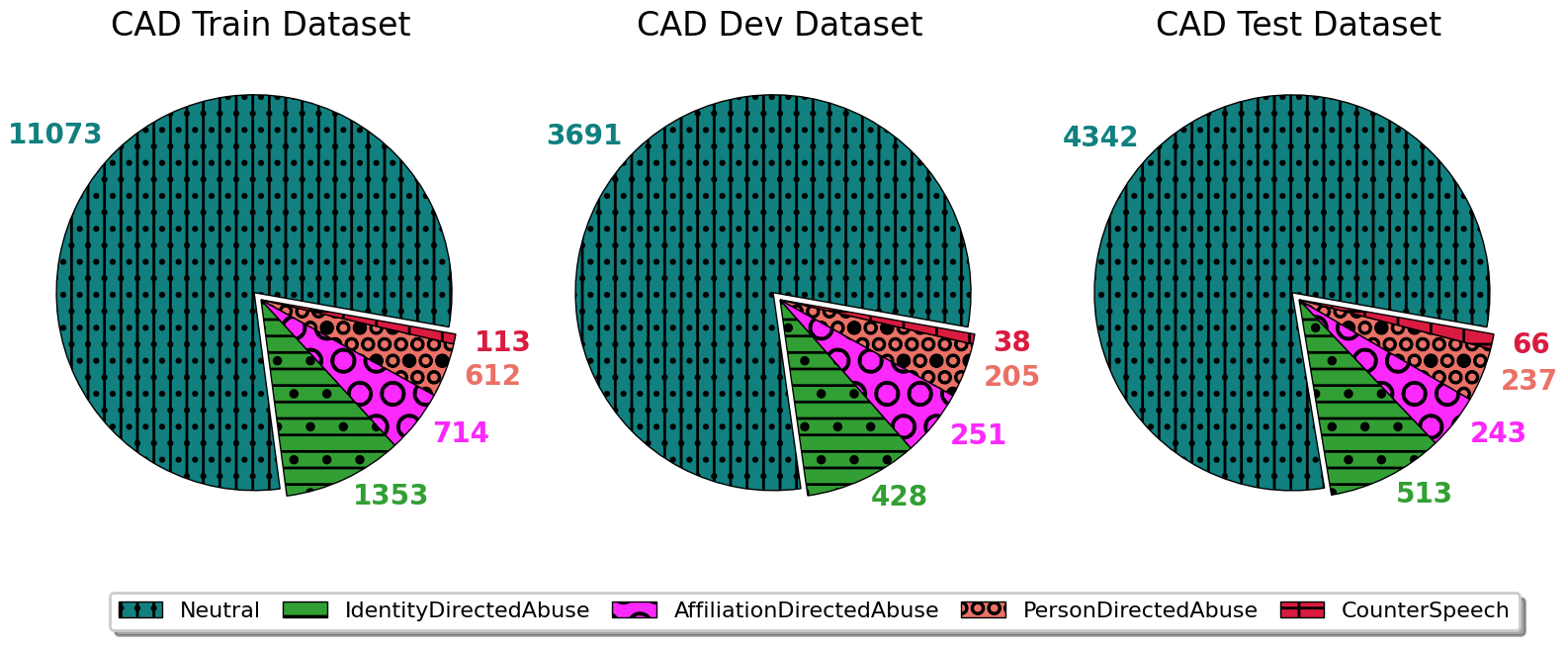} 
    } 
\end{figure}

\begin{figure}[!t]
    \centering
    \caption{\textbf{Dataset for B-C: } In the figure below, we see the CAD {\fontfamily{qcr}\selectfont train}, {\fontfamily{qcr}\selectfont dev}  and {\fontfamily{qcr}\selectfont test} dataset with number of samples for "Non-hateful" and "Hateful" category. The ratio between "Non-hateful" to "Hateful" category is approximately $80\,\%$ to $20\,\%$. For the dataset~\cite{rottger2020hatecheck}, the ratio is approximately $32\,\%$ to $68\,\%$.}
    \label{fig:cad_hc_bc_dataset}
    {
    \includegraphics[width=1.\textwidth]{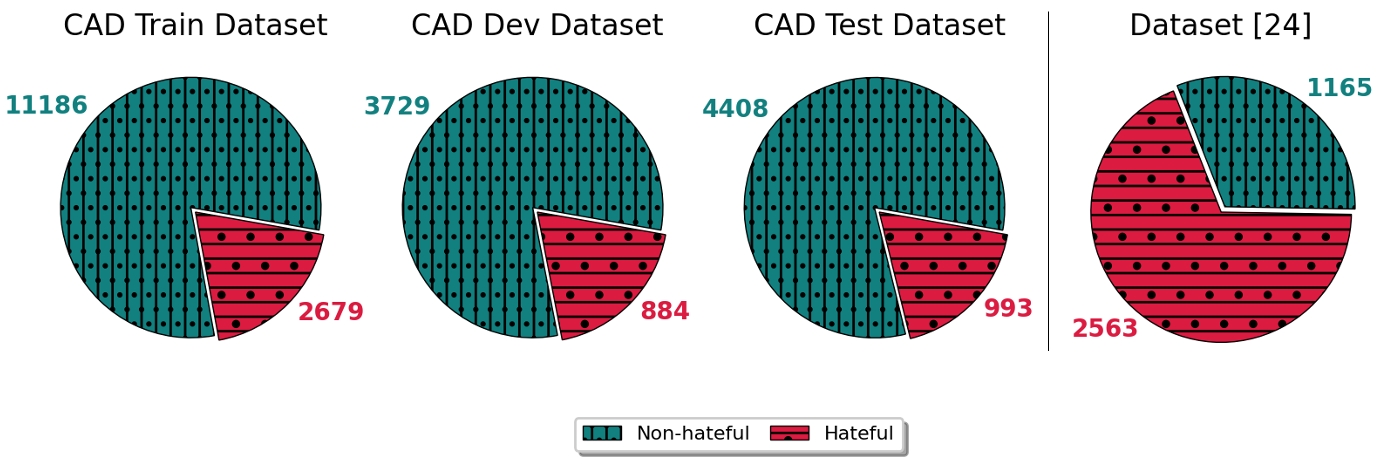} 
    } 
\end{figure}

Finally, we assess the performance of MLT-plus-TM, trained with CAD data, on a distinct high-quality human-labeled dataset~\cite{rottger2020hatecheck}. 

The dataset~\cite{rottger2020hatecheck} encompasses an extensive array of functional tests within both the "Hateful" and "Non-hateful" categories. We have confirmed with the authors of CAD~\cite{vidgen2021introducing} and \cite{rottger2020hatecheck} that these two datasets were independently constructed without any overlap between them. However, it's noteworthy that both datasets adhere to the same well-defined OOPS taxonomy, rendering them suitable for rigorous assessment of our binary classifier (~\acrshort{b-c}) through transfer learning.

Significantly, the two datasets exhibit distinctly different class distributions. In CAD, the ratio between "Hateful" and "Non-hateful" instances stands at approximately $20\,\%$ to $80\,\%$. This distribution aligns more closely with real-world scenarios where "Non-hateful" content predominates. In contrast, the distribution in \cite{rottger2020hatecheck} is skewed in the opposite direction, with a ratio of $68.8\,\%$ "Non-hateful" to $31.2\,\%$ "Hateful". It's important to note that the dataset~\cite{rottger2020hatecheck} is designed to uncover model weaknesses, intentionally featuring a considerably larger proportion of challenging cases than they occur naturally. Consequently, this skewed distribution in \cite{rottger2020hatecheck} is not reflective of a real-world class distribution. In situations where such distribution discrepancies exist between training and test data, the threshold-moving (TM) technique might shift its "favor" in a direction contrary to its intended effect. We offer illustrative examples and delve into this topic further in Section~\ref{sec:pbi}.



\subsection{MLT architecture} \label{sec:combiner}
The MLT combiner architecture consists of the following components: (i) individual base models characterized by the embeddings they use, and ii) the method for combining the models.  
\begin{enumerate}
\def\labelenumi{\arabic{enumi}.}
\item
  Individual base models
  \begin{itemize}
      \item \textbf{BERT}~\cite{devlin2018bert}\textbf{:} It is one of the state-of-the-art model for solving various tasks in the area of natural language processing. We use an instance from~\cite{BERTEncoder} pre-trained on real-world text from Wikipedia (en.wikipedia.org) and BooksCorpus~\cite{Wolfram:BookCorpus}.
      \item \textbf{BERTweet}~\cite{qudar2020tweetbert}\textbf{:} Based on BERT, it is the first public large-scale pre-trained language model for English Tweets. It was pre-trained on $850$M English Tweets (cased).
      \item \textbf{Bigbird}~\cite{zaheer2020big}\textbf{:} It is a sparse-attention-based transformer model that extends BERT to much longer sequences. 
      \item \textbf{Bloom}~\cite{scao2022bloom}\textbf{:} It emerged as world's largest open multi-lingual language model. Bloom has $176$B parameters and is able to generate text in $46$ languages.
      \item \textbf{XLNet}~\cite{yang2019xlnet}\textbf{:} It emerged as an extension to Transformer-XL~\cite{dai2019transformer}, since it borrows the recurrence mechanism from Transformer-XL, to establish long-term dependencies and, hence, XLNet is a popular choice in the area of long-range NLP tasks.
  \end{itemize}
\end{enumerate}

\begin{quote}

All five individual models were created using the AutoModelForSequenceClassification from the Hugging Face Transformers library~\cite{wolf2019huggingface}. This is a generic model class designed for loading pre-trained transformer models. It additionally incorporates a classification head on top of the model outputs, specifically utilized for tasks involving classification. The same CAD 
{\fontfamily{qcr}\selectfont train} data were used as training data; however, they are utilized at distinct level of granularity for~\acrshort{m-c} and~\acrshort{b-c}.
\end{quote}

\begin{enumerate}
\def\labelenumi{\arabic{enumi}.}
\setcounter{enumi}{1}
\item
  Combiner Method
\end{enumerate}
\label{subsec:combiner-method}

\begin{quote}
We subsequently construct~\acrshort{m-c} 
and~\acrshort{b-c} MLT combiners by amalgamating five individual models.
In both classifications, the predicted probabilities on
CAD {\fontfamily{qcr}\selectfont dev} data from the individual models are used as input
within a neural network architecture. The hidden layer has only one node responsible for
combining the input values. The resulting output entails
a consolidated, singular predicted probability value, post application of sigmoid activation, in accordance with equation~\eqref{eq:combinedpredictor}-\eqref{eq:biasedsigma}.

In~\acrshort{b-c} MLT, only a single combiner on "Hateful" class is necessary. The outcome corresponds to the merged probability of the "Hateful" class.

Multi-label classification algorithms can be categorized
into two different groups~\cite{Trohidis2011-lz}: (i) problem transformation
methods, and (ii) algorithm adaptation methods. 
Binary relevance (BR)~\cite{Tsoumakas2010MiningMD,Zhang2018-lh} is a popular problem transformation approach that trains binary classifiers, one for each class in the datasets. One vs. the rest (OVR) is a typical type of the BR approach. BR has faced criticism due to its assumption of label independence within a multi-label dataset, disregarding potential correlations
among labels. Nevertheless, despite this constraint, the BR method remains a straightforward and efficient strategy for
tackling multi-label classification challenges, alongside the methods in the second group that directly handle multi-label data.

Within the CAD dataset~\cite{vidgen2021introducing}, only $1.94\,\%$ of entries in the training
set contain more than one primary category. The training data have a label cardinality
of $1.02$. Unlike many other multi-label classifications where entries with multiple categories constitute a significant portion of the data, the label correlations within CAD are not dominantly strong. 

In~\acrshort{m-c} MLT, we extend the aforementioned classification task approaches to construct our MLT combiners.
We explored both strategies in~\acrshort{m-c} MLT: a BR-like method where we decompose the $N$ columns and create separate combiners for each column, and a neural network approach to build a single combiner that learns a direct mapping from probabilities across all $N$ columns to the output of combined probabilities for all $N$ columns in simultaneously.

Throughout our experiments, both methods yielded comparable performance. Despite the fact that the second method captures label correlations, the BR method, which assumes label independence, performs comparatively well due to the label structure of the CAD dataset. We choose the BR method as our final choice in MLT due to its simplicity, efficiency, and its seamless integration with training threshold cut-off values later on – a method also based on BR principles.

In the BR method, we employ an \textit{$N$-separate combiners
approach}: we decompose the $N$ columns and construct an individual combiner for each of them.
Every combiner merges the predicted probabilities from the five base models for a particular column into a unified result.
This process is iterated for all $N$ = $5$ columns within the CAD dataset, yielding the following combiners: "Neutral" combiner, "Identity-directed abusive" combiner,
"Affiliation-directed abuse" combiner, "Person-directed abuse" combiner, and "Counter Speech" combiner. 
The composition of verdicts from all binary combiners constitutes the multi-label output.

\end{quote}

\begin{figure}[!t]
    \centering
    \caption{\textbf{Class hierarchy} (Adapted from~\cite{vidgen2019challenges})}
    \label{fig:class_hierarchy}
    {
    \includegraphics[width=0.40\textwidth]{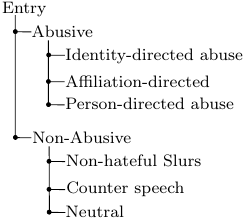}
    } 
\end{figure}

\subsection{Data Composition}
\label{sec:data_composition}
CAD is a high-quality labeled dataset with a granular level and with a great coverage of both "Abusive" and "Non-abusive" sub-categories, cf., Figure~\ref{fig:class_hierarchy}.
We use five sub-categories in our~\acrshort{m-c}, i.e., we exclude "Non-hateful Slurs".
"Neutral" entries dominate, accounting for $79.8\,\%$
of the data. "Identity-directed Abuse" accounts for $9.9\,\%$, "Affiliation-directed Abuse"
($5.0\,\%$), "Person-directed Abuse" ($4.0\,\%$) and "Counter Speech" ($0.8\,\%$).
"Non-hateful Slurs" accounts for only $0.5\,\%$ and is removed from consideration~\cite{vidgen2021introducing}. The distribution of the five sub-categories that we use from CAD dataset in M-C, for each of the {\fontfamily{qcr}\selectfont train}, {\fontfamily{qcr}\selectfont dev} and {\fontfamily{qcr}\selectfont test} splits is shown in Figure~\ref{fig:cad_mc_dataset}.

In the case of B-C, for the CAD dataset, the top-level entries of "Abusive" and "Non-abusive" are mapped to "Hateful" and "Non-hateful" classes, i.e., for each of the five sub-categories, we use "Neutral", and "Counter Speech" as the "Non-hateful" category and "Identity-directed Abuse", "Affiliation-directed Abuse" and "Person-directed Abuse" as the "Hateful" category, respectively (cf., Figure~\ref{fig:cad_hc_bc_dataset}). After the mapping, the ratio between the two classes is about $20\,\%$ to $80\,\%$ for all the three {\fontfamily{qcr}\selectfont train}, {\fontfamily{qcr}\selectfont dev} and {\fontfamily{qcr}\selectfont test} splits.

CAD is used in all aspects of training for both ~\acrshort{m-c} and ~\acrshort{b-c} including individual models, MLT combining weights, and threshold cut-off values for TM.

The independent dataset~\cite{rottger2020hatecheck} is used to test above trained systems.
The ratio between "Hateful" to "Non-hateful" categories in the dataset~\cite{rottger2020hatecheck} is around $68\,\%$ to $32\,\%$ (cf., Figure~\ref{fig:cad_hc_bc_dataset}). Unlike in CAD where most of samples are "Non-hateful" as in real-world situation, the dataset~\cite{rottger2020hatecheck} has a much larger portion of "Hateful" class, but it is deliberately crafted that way for testing system performance purpose. 

Both CAD and \cite{rottger2020hatecheck} are high quality and share consistent OOPS interpretation.

\subsection{Threshold-moving (TM)} \label{sec:tm}
Models trained to minimize  the overall loss function with imbalanced data tend to exhibit bias in favor of the majority class,
using the default $0.5$ threshold to map probabilities to category labels would more likely assign minority class data to the majority class. Threshold-moving (TM) is utilized to counter this bias by enhancing the likelihood of correctly classifying minority data. An optimal threshold value is trained by optimizing F1 on the positive class (minority class) "Hateful" in~\acrshort{b-c} after MLT.

In the case of~\acrshort{m-c}, employing the default threshold of 0.5 would assign the predicted class as the one with the highest probability equal to or greater than 0.5. Various TM strategies are available. The simplest approach is training a single global threshold value for all classes to achieve the best performance, known as the score-cut strategy (S-Cut). The default threshold value of 0.5 is a special instance of S-Cut. S-Cut returns a label set where scores exceed or equal the trained threshold value. The rank-cut strategy (R-cut) aims to return a set of precisely n labels with the highest scores. However, in our application, most samples possess only one class, with only a few having more than one. Thus, R-cut is not a suitable candidate. A more complex strategy, called CS-Cut, exists. The class-specific score-cut (CS-Cut)~\cite{6577846} optimizes a vector of thresholds, each corresponding to a label. It then employs the same thresholding operation as S-Cut. CS-cut proves useful when certain labels are harder to classify than others, as is the case in this application. Less dominant classes are more challenging to classify accurately. Hence, we adopt this method, as depicted in Algorithm~\ref{algo:tm}. CS-cut can provide higher accuracy but is more susceptible to overfitting compared to S-Cut.

Within CS-Cut, multiple variants exist, ranging from simple to more intricate approaches. In CS-Cut for~\acrshort{m-c} thresholding, we adopt a binary relevance (BR)-like method. BR is among the most common thresholding methods in multi-label classification~\cite{6471714}. It's both simple and highly efficient, serving as a strong baseline for comparison with more advanced methods. Given our use of BR in MLT, training threshold values and applying them through TM on MLT output are streamlined and straightforward.
 
As elaborated in Section~\ref{subsec:combiner-method} for~\acrshort{m-c}, five MLT combiners generate predicted probabilities for five classes. Post-MLT, we train a threshold value for TM using the MLT output for each class, similar to~\acrshort{b-c}. We repeat this process for all classes. The resultant final threshold values for all classes constitute a tuple of five values ranging between 0 and 1. This TM approach assumes no class dependency in~\acrshort{m-c} regarding the effect of a threshold value in one class on those from other classes. Therefore, extending TM training from~\acrshort{b-c} to~\acrshort{m-c} is straightforward. This assumption of class independence greatly simplifies the task of identifying suitable threshold values for~\acrshort{m-c}. The macro-averaged F1 score for~\acrshort{m-c} is the average of F1 scores calculated in~\acrshort{b-c} for each class.

TM alters the decision-making dynamics between the minority and majority classes and impacts evaluation results. Depending on the most important for the application, finding the optimal threshold values becomes necessary to maximize performance with respect to the chosen metric, or to satisfy lower bound constraints on other metrics while optimizing the primary one.
 
In our experiments, TM improves the \textit{macroF1} score in both~\acrshort{m-c} ($0.439$ to $0.467$) 
and~\acrshort{b-c} ($0.714$ to $0.722$) on CAD {\fontfamily{qcr}\selectfont test} data, at the expense of a slight accuracy decline. This aligns with expectations, as further discussed in Section~\ref{sec:combinerPerformance}. We also applied TM trained on CAD to a different dataset~\cite{rottger2020hatecheck} in ~\acrshort{b-c}. 
Through this transfer learning, we achieve performance improvements in both \textit{macroF1} ($0.688$ to $0.69$) and accuracy ($0.717$ to $0.741$). More extensive details and analysis are provided in Section~\ref{sec:combinerPerformance}.

\begin{algorithm}[!t] 
    \footnotesize
    \caption{Optimal threshold training in TM}
    \label{algo:tm} 
    \textbf{Input: }classes, step size, base and max threshold
    \begin{algorithmic}[1] 
        \State $\mathbf{I}^{(C)} \leftarrow $ number of classes \textsl{\Comment{Initialization}}
        \State $\delta \leftarrow \text{step size}$, $\alpha \leftarrow \text{base threshold}$, $\beta \leftarrow \text{max threshold}$ 
        \State $\mathbb{D} \leftarrow$ predicted probabilities from individual model
        \State execute Meta Learning Technique on $\mathbb{D}$ \textsl{\Comment{MLT}}
        \State $\mathcal{O}pt\mathcal{T}h \leftarrow \phi$ \textsl{\Comment{Stores optimal threshold for each class}}
        \For{$i$ $=$ $1$ to $\mathbf{I}^{(C)}$}
            \State $\textsl{DICT}^{F_1}$ $\leftarrow \phi$
            \State $\mathbb{D}'[i] \leftarrow \text{build data from $\mathbb{D}$ corresponding to}$ $i^{th}$ class
            \For{$\Delta$ $=$ $\alpha$ to $\beta$}  \textsl{\Comment{TM}}
                \State $\gamma \leftarrow$ $F_1$ score for $\mathbb{D}'[i]$ using $\Delta$ as threshold
                \State $\textsl{DICT}^{F_1}[\Delta] \leftarrow \gamma$
                \State $\Delta \leftarrow \Delta + \delta$
            \EndFor
            \State $\mathcal{O}pt\mathcal{T}h[i]$ $\leftarrow$ $key(max(\textsl{DICT}^{F1}))$
        \EndFor
        \State \textbf{Return} $(\mathcal{O}pt\mathcal{T}h)$ \textsl{\Comment{Return optimal threshold}}
    \end{algorithmic} 
\end{algorithm}	

\section{Computational Results}
\label{sec:results}
In this section, we provide computational results with analysis. We use macro-averaged F1 (\textit{macroF1}) score as the main metric in performance evaluation because datasets in the OOPS domain are highly imbalanced, and we want to treat all classes equally, regardless of their presence volume in the datasets. We also provide accuracy in reporting as a reference, but it should be considered a secondary metric.

All results are reported as means with standard deviations from multiple runs ($n$ represents the number of runs in all tables). In case of individual models, two runs from an individual model are used in performance evaluation and in joining combiners. In case of combiners (MLT and MLT-plus-TM runs), if each individual model has two runs,  combining five individual models with all possible permutations create 32 runs.

Three sets of results are presented: ~\acrshort{m-c} (Table ~\ref{tab:cad_mc}) and ~\acrshort{b-c} (Table ~\ref{tab:bc_on_cad}) on CAD test, and ~\acrshort{b-c} (Table ~\ref{tab:bc_on_hatecheck}) on dataset~\cite{rottger2020hatecheck} using transfer learning.
For both ~\acrshort{m-c} and ~\acrshort{b-c} on CAD, five individual models are trained using the CAD training dataset. Subsequently, predictions are obtained on the CAD dev dataset and used to train combining weights through the MLT structure explained in Section~\ref{sec:combiner}. Following the MLT step, threshold cut-off values are trained (detailed in Section~\ref{sec:tm}) using the output from MLT. The cut-off values are then applied through TM on the predictions from the CAD test dataset to obtain the final MLT-plus-TM results.

For ~\acrshort{m-c}, we use the reported performance on the same tasks in ~\cite{ rottger2020hatecheck} as baselines to compare with our results.

In ~\acrshort{b-c} transfer learning, we utilize the trained MLT combining weights and TM cut-off values from CAD to generate MLT and MLT-plus-TM results on dataset~\cite{rottger2020hatecheck}. Finally, we compare our results with the reported ones in ~\cite{rottger2020hatecheck} and provide observations and analysis. 

All experiments reported here were developed in Python 3.8 with the TensorFlow 2.9.1 ~\cite{tensorflow} library and PyTorch Lightning (version 1.7.7). The source code for reproducing the results reported here along with appropriate meta data is available at \url{https://data.nist.gov/od/id/mds2-3074}. 

The individual participating open models, \textbf{BigBird} (M1), \textbf{BERT} (M2), \textbf{BERTweet} (M3), \textbf{Bloom} (M4), and \textbf{XLNet} (M5), were trained on a professional Graphics Processing Unit (GPU) cluster with 8 NVIDIA Tesla V100 (32 GB each).
The experiments for \textbf{MLT} and \textbf{TM} were carried on a 2017 MacBook Pro with 3.1 GHz Quad-Core Intel Core i7 and 16 GB RAM, {\em without} Graphics Processing Unit (GPU) acceleration.

\subsection{Performance of~\acrshort{m-c}}
\label{sec:pmc}

Table~\ref{tab:cad_mc} presents the performance at all category levels. 
Among the five participating individual models, BERTweet demonstrates the best overall performance, achieving the highest accuracy and \textit{macroF1} scores.
MLT surpasses all individual models in accuracy and only slightly lags behind BERTweet in \textit{macroF1} score. MLT-plus-TM further improves the \textit{macroF1} score to the highest level, albeit at the expense of a slight decrease in accuracy.
However, in imbalanced classification, relying solely on accuracy can be misleading as it may not adequately account for minority classes. On the other hand, \textit{macroF1} score is a more meaningful metric in an imbalanced classification settings.
Our MLT-plus-TM outperforms the baseline in both \textit{macroF1} score ($0.467$ vs. $0.455$) and accuracy ($0.780$ vs. $0.769$) and achieves the highest \textit{macroF1} score in all runs.

Table~\ref{tab:cad_mc_per_category} presents the performance at a per-category level.  MLT-plus-TM achieves the highest F1 score in all categories, except for `Counter Speech'. However, it is worth noting that the 
`Counter Speech' category has a very low prevalence ($0.8\,\%$), which makes the performance in this category less reliable and should be interpreted with caution.

\begin{table}[!t]
\footnotesize
\centering
    \caption{\textbf{CAD~\acrshort{m-c} scores on the {\fontfamily{qcr}\selectfont test} set}}
  \label{tab:cad_mc}
  {
  \begin{tabular}{|p{6cm}||p{4cm}|p{4cm}|}
    \hline
    \textbf{Model}(\#) & \textbf{Accuracy} & \textbf{\textit{macroF1}} \\ 
    \hline \hline
    \textbf{DistilBERT Baseline 1 (n=5)} & 0.769 (0.005)  & 0.44 (0.007) \\
    \hline
    \textbf{BERT Baseline 2 (n=5)} & 0.762 (0.005)  & 0.455 (0.006) \\
    \hline
    \textbf{Bigbird (n=2)} & 0.806(0.0076) & 0.418(0.0086) \\
    \hline
    \textbf{BERT (n=2)} & 0.806(0.0041) & 0.424(0.0025) \\
    \hline
    \textbf{BERTweet (n=2)} & 0.806(0.0043) & 0.441(0.0024) \\
    \hline
    \textbf{Bloom (n=2)} & 0.785(0.0075) & 0.304(0.0035) \\
    \hline
    \textbf{XLNet (n=2)} & 0.802(0.0039) & 0.439(0.0014) \\
    \hline
    \textbf{MLT combiner (n=32)} & 0.821(0.0031) & 0.439(0.0049) \\
    \hline
    \textbf{MLT-plus-TM (n=32)} & 0.780(0.0096) & 0.467(0.0095)\\
    \hline
  \end{tabular}
  }
\end{table}

\begin{table*}[h!]
    \caption{\textbf{CAD~\acrshort{m-c} Scores per category on the {\fontfamily{qcr}\selectfont test} set. Scores: P=Precision, R=Recall, F1=F1}}
  \label{tab:cad_mc_per_category}
  \footnotesize
  \resizebox{\textwidth}{!}{
  \begin{tabular}{|p{2cm}||p{1cm}|p{1cm}|p{1cm}|p{1cm}|p{1cm}|p{1cm}|p{1cm}|p{1cm}|p{1cm}|p{1cm}|p{1cm}|p{1cm}|p{1cm}|p{1cm}|p{1cm}|}
    \hline
    \multirow{3}{*}{\textbf{Model}(\#)} & \multicolumn{3}{|c|}{\textbf{Neutral}} & \multicolumn{3}{|c|}{\textbf{Identity-directed}} &  \multicolumn{3}{|c|}{\textbf{Affiliation-directed}} & \multicolumn{3}{|c|}{\textbf{Person-directed}} &  \multicolumn{3}{|c|}{\textbf{Counter Speech}}\\
    & \multicolumn{3}{|c|}{N=4342} & \multicolumn{3}{|c|}{N=513} &  \multicolumn{3}{|c|}{N=243} & \multicolumn{3}{|c|}{N=237} &  \multicolumn{3}{|c|}{ N=66}\\ 
    \cline{2-16}
    & P & R & F1 & P & R & F1 & P & R & F1 & P & R & F1 & P & R & F1\\ 
    \hline \hline
    \textbf{DistilBERT Baseline 1} (n=5) & 0.88 & 0.917 & 0.898 & 0.414 & 0.473 & 0.441 & 0.368 & 0.45 & 0.405 & 0.359 & 0.404 & 0.38 & 0.083 & 0.073 & 0.076\\
    \hline
    \textbf{BERT Baseline 2} (n=5) & 0.883 & 0.922 & 0.902 & 0.411 & 0.51 & 0.455 & 0.368 & 0.481 & 0.416 & 0.356 & 0.488 & 0.411 & 0.107 & 0.088 & 0.091\\
    \hline
    \textbf{Bigbird} (n=2) & 0.940 & 0.871 & 0.904 & 0.298 & 0.621 & 0.402 & 0.327 & 0.404 & 0.361 & 0.317 & 0.481 & 0.379 & 0.030 & 0.084 & 0.044\\
    \hline
    \textbf{BERT} (n=2) & 0.933 & 0.876 & 0.903 & 0.322 & 0.581 & 0.414 & 0.333 & 0.497 & 0.396 & 0.285 & 0.467 & 0.354 & 0.038 & 0.179 & 0.053\\
    \hline
    \textbf{BERTweet} (n=2) & 0.926 & 0.882 & 0.904 & 0.356 & 0.547 & 0.429 & 0.377 & 0.466 & 0.416 & 0.382 & 0.441 &  0.409 & 0.030 & 0.129 & 0.046\\
    \hline
    \textbf{Bloom} (n=2) & 0.964 & 0.845 & 0.901 & 0.157 & 0.527 & 0.233 & 0.191 & 0.367 & 0.245 & 0.122 & 0.239 &  0.142 & 0.0000 & 1.0000 & 0.0000\\
    \hline
    \textbf{XLNet} (n=2) & 0.927 & 0.876 & 0.901 & 0.386 & 0.517 & 0.442 & 0.329 & 0.470 & 0.387 & 0.325 & 0.469 &  0.383 & 0.068 & 0.109 & 0.083\\
    \hline
    \textbf{MLT combiner} (n=32) &  0.871 & 0.960 & 0.913 &  0.678 &  0.338 & 0.451 & 0.573 & 0.343 & 0.429 & 0.609 & 0.302 &  0.403 & 1.0000 & 0.0000 & 0.0000\\
    \hline
    \textbf{MLT-plus-TM} (n=32) & 0.866 & 0.968 & 0.914 & 0.504 & 0.478 & 0.488 & 0.418 & 0.476 & 0.443 & 0.471 & 0.427 & 0.441 & 0.269 & 0.046 & 0.050\\
    \hline
  \end{tabular}
  }
\end{table*}

\subsection{Performance of~\acrshort{b-c}} \label{sec:pbi}
Table~\ref{tab:bc_on_cad} displays the performance on CAD test. Once again, BERTweet emerges as the top-performing individual model. But MLT outperforms all individual models in terms of both \textit{macroF1} score and accuracy. MLT-plus-TM further enhances the \textit{macroF1} score to
its highest level, albeit with a slight decrease in accuracy. Both ~\acrshort{m-c} and ~\acrshort{b-c} on CAD test exhibit very similar patterns in TM.

Table~\ref{tab:bc_on_hatecheck} presents the performance on dataset~\cite{rottger2020hatecheck}, which represents a transfer learning scenario involving an out-of-distribution target dataset. Our objective is to assess how well MLT and TM trained from CAD generalize to dataset~\cite{rottger2020hatecheck}. We believe that models trained with 
CAD data tend to capture the real-world imbalanced distribution between the `Hateful' and "Non-hateful" classes,
which may introduce some bias towards predicting the majority  "Non-hateful" class. Therefore, TM is applied to counterbalance this bias. It is important to note that dataset~\cite{rottger2020hatecheck} deliberately has a much larger portion of the 'Hateful' class, resulting in an opposite class distribution compared to CAD. This deliberate design choice serves the purpose of testing and evaluation.

\begin{table*}[h!]
    \caption{\textbf{B-C performance on CAD test data. Scores: Accuracy (Recall) and F1 with Std by test case label}}
  \label{tab:bc_on_cad}
  \footnotesize
  \resizebox{\textwidth}{!}{
  \begin{tabular}{|c||c|c|c|c|c|c|}
    \hline
    \multirow{2}{*}{\textbf{Model}(\#)} & \multicolumn{2}{|c|}{\textbf{Hateful} N=899} & \multicolumn{2}{|c|}{\textbf{Non-hateful} N=4408} & \multicolumn{2}{|c|}{\textbf{Total} N=5307}\\
    \cline{2-7}
    & Accuracy(Recall) & F1 & Accuracy(Recall) & F1 & Accuracy(Recall) & \textit{macroF1}\\
    \hline \hline
    \textbf{Bigbird (n=2)} & 0.572(0.0047) & 0.460(0.0001) & 0.882(0.0002) & 0.911(0.0006) & 0.847(0.0008) & 0.685(0.0003)\\
    \hline
    \textbf{BERT (n=2)} & 0.546(0.0145) & 0.492(0.0056) &  0.891(0.0004) & 0.907(0.0024) & 0.843(0.0037) & 0.700(0.0040)\\
    \hline
    \textbf{BERTweet (n=2)} & 0.580(0.0421) & 0.510(0.0260) & 0.895(0.0099) & 0.911(0.0056) & 0.850(0.0068) & 0.711(0.0102)\\
    \hline
    \textbf{Bloom (n=2)} & 0.499(0.0430) & 0.372(0.0653) & 0.870(0.0131) & 0.899(0.0082) & 0.827(0.0096) & 0.635(0.0286)\\
    \hline
    \textbf{XLNet (n=2)} & 0.531(0.0260) & 0.489(0.0028) & 0.892(0.0029) & 0.904(0.0048) & 0.839(0.0067) & 0.697(0.0010)\\
    \hline
    \textbf{MLT (n=32)} & 0.430(0.0172) & 0.510(0.0081) & 0.948(0.0055) & 0.918(0.0016) & 0.860(0.0021) & 0.714(0.0038)\\
    \hline
    \textbf{MLT-plus-TM(n=32)} & 0.500(0.0304) & 0.533(0.0080) & 0.923(0.0122) & 0.912(0.0039) & 0.851(0.0052) & 0.722(0.0031)\\
    \hline
  \end{tabular}
  }
\end{table*}

\begin{table*}[h!]
    \caption{\textbf{B-C performance on dataset~\cite{rottger2020hatecheck}. Scores: Accuracy (Recall) and F1 with Std by test case label}}
  \label{tab:bc_on_hatecheck}
  \footnotesize
  \resizebox{\textwidth}{!}{
  \begin{tabular}{|c||c|c|c|c|c|c|}
    \hline
    \multirow{2}{*}{\textbf{Model}(\#)} & \multicolumn{2}{|c|}{\textbf{Hateful} N=2563} & \multicolumn{2}{|c|}{\textbf{Non-hateful} N=1165} & \multicolumn{2}{|c|}{\textbf{Total} N=3728}\\
    \cline{2-7}
    & Accuracy(Recall) & F1 & Accuracy(Recall) & F1 & Accuracy(Recall) & \textit{macroF1}\\
    \hline \hline
    \textbf{B-D} & 0.755 & 0.738 & 0.36 & 0.379 & 0.632 & 0.559\\
    \hline
    \textbf{B-F} & 0.655 & 0.694 & 0.485 & 0.432 & 0.602 & 0.563\\
    \hline
    \textbf{Google-P} & 0.895 & 0.841 & 0.482 & 0.563 & 0.766 & 0.702\\
    \hline
    \textbf{Bigbird (n=2)} & 0.847(0.0021) & 0.559(0.0120) & 0.394(0.0033)  & 0.535(0.0014) & 0.547(0.0070) & 0.547(0.0067)\\
    \hline
    \textbf{BERT (n=2)} & 0.795(0.0028) & 0.766(0.0183) &  0.506(0.0233) & 0.539(0.0012) & 0.690(0.0164) & 0.653(0.0097)\\
    \hline
    \textbf{BERTweet (n=2)} & 0.805(0.0099) & 0.810(0.0374) & 0.613(0.0924) & 0.573(0.0018) & 0.739(0.0363) & 0.692(0.0196)\\
    \hline
    \textbf{Bloom (n=2)} & 0.696(0.0147) & 0.156(0.0475) & 0.313(0.0011) & 0.466 (0.0058) & 0.346(0.0099) & 0.311(0.0208)\\
    \hline
    \textbf{XLNet (n=2)} & 0.776(0.0128) & 0.731(0.0249) & 0.455(0.0147) & 0.498(0.0185) & 0.651(0.0166) & 0.614(0.0032)\\
    \hline
    \textbf{MLT (n=32)} & 0.743(0.0321) & 0.783(0.0149) & 0.660(0.0325) & 0.593(0.0065) & 0.717(0.0128) & 0.688(0.0081)\\
    \hline
    \textbf{MLT-plus-TM(n=32)} & 0.833(0.0346) & 0.815(0.0106) & 0.540(0.0597) & 0.565(0.0255) & 0.741(0.0087) & 0.690(0.0109)\\
    \hline
  \end{tabular}
  }
\end{table*}

For ease of comparison, we include the following model performances reported in 
\cite{rottger2020hatecheck} in Table ~\ref{tab:bc_on_hatecheck}: 
two pre-trained BERT models denoted as B-D and B-F, as well as one 
commercial non-open model - Google Jigsaw's Perspective - denoted as Google-P.
It is important to note that these models have been trained on different data; however, they are all tested on the same suite~\cite{rottger2020hatecheck}.

Table~\ref{tab:bc_on_hatecheck} illustrates that MLT slightly underperforms when compared to the best individual model, BERTweet, in transfer learning.  We provide further analysis on this in Section~\ref{sec:mlt-analysis}. However, MLT-plus-TM achieves the best performance in terms of accuracy and performs nearly as well as the best model in terms of \textit{macroF1} score.

At the category level, in the "Hateful" category (Table~\ref{tab:bc_on_hatecheck}), MLT-plus-TM outperforms all individual participating models in terms of F1 score but falls slightly behind the Google-P model ($0.815$ vs. $0.841$). In the "Non-hateful" category (Table~\ref{tab:bc_on_hatecheck}), MLT outperforms all other models, including the Google-P model ($0.593$ vs. $0.563$), in terms of F1 score. MLT-plus-TM also marginally outperforms the Google-P model ($0.565$ vs. $0.563$).

Overall, based on Table~\ref{tab:bc_on_hatecheck}, we can observe that both MLT ($0.688$) and MLT-plus-TM ($0.690$) exhibit performance almost comparable to that of the Google-P model ($0.702$) in \textit{macroF1}.

\begin{table}[!ht]
    \centering
    \caption{\textbf{MLT M-C performance in \textit{macroF1} scores order on CAD test}}
    \begin{tabular}{|l|l|l|}
    \hline
        \textbf{Model Combination} & \textbf{\textit{macroF1}} & \textbf{Accuracy} \\ \hline
        M1, M2, M3, M4, M5 & 0.439 & 0.821 \\ \hline
        M1, M2, M3, M5 & 0.438 & 0.821 \\ \hline
        M1, M3, M4, M5 & 0.432 & 0.820 \\ \hline
        M1, M2, M3, M4 & 0.431 & 0.821 \\ \hline
        M1, M2, M3 & 0.431 & 0.824 \\ \hline
        M2, M3, M4, M5 & 0.429 & 0.819 \\ \hline
        M1, M3, M5 & 0.429 & 0.821 \\ \hline
        M1, M2, M4, M5 & 0.428 & 0.820 \\ \hline
        M2, M3, M5 & 0.427 & 0.819 \\ \hline
        M1, M2, M5 & 0.425 & 0.822 \\ \hline
        M2, M3, M4 & 0.410 & 0.819 \\ \hline
        M3, M4, M5 & 0.407 & 0.819 \\ \hline
        M1, M3, M4 & 0.406 & 0.816 \\ \hline
        M2, M3 & 0.406 & 0.817 \\ \hline
        M1, M3 & 0.404 & 0.823 \\ \hline
        M3, M4 & 0.403 & 0.803 \\ \hline
        M3, M5 & 0.402 & 0.820 \\ \hline
        M1, M2, M4 & 0.401 & 0.820 \\ \hline
        M2, M4, M5 & 0.401 & 0.820 \\ \hline
        M1, M4, M5 & 0.399 & 0.814 \\ \hline
        M1, M5 & 0.392 & 0.814 \\ \hline
        M2, M5 & 0.391 & 0.805 \\ \hline
        M1, M2 & 0.391 & 0.826 \\ \hline
        M2, M4 & 0.384 & 0.805 \\ \hline
        M4, M5 & 0.380 & 0.803 \\ \hline
        M1, M4 & 0.368 & 0.806 \\ \hline
        \textbf{Average score} & \textbf{0.410} & \textbf{0.817} \\ \hline
    \end{tabular}
    \label{tab:mlt_mc}
\end{table}

\begin{table}[!ht]
    \centering
    \caption{\textbf{MLT-plus-TM M-C performance in \textit{macroF1} scores order on CAD test}}
    \begin{tabular}{|l|l|l|}
    \hline
        \textbf{Model Combination} & \textbf{\textit{macroF1}} & \textbf{Accuracy} \\ \hline
        M2, M3, M5 & 0.469 & 0.778 \\ \hline
        M1, M2, M3, M4, M5 & 0.467 & 0.780 \\ \hline
        M1, M3, M4, M5 & 0.465 & 0.780 \\ \hline
        M1, M2, M3, M5 & 0.464 & 0.784 \\ \hline
        M2, M3, M4, M5 & 0.464 & 0.784 \\ \hline
        M3, M4, M5 & 0.462 & 0.783 \\ \hline
        M1, M3, M5 & 0.461 & 0.787 \\ \hline
        M1, M2, M4, M5 & 0.459 & 0.786 \\ \hline
        M1, M2, M3, M4 & 0.459 & 0.786 \\ \hline
        M3, M5 & 0.459 & 0.789 \\ \hline
        M1, M2, M3 & 0.457 & 0.787 \\ \hline
        M1, M3 & 0.457 & 0.782 \\ \hline
        M2, M3 & 0.456 & 0.783 \\ \hline
        M1, M3, M4 & 0.456 & 0.783 \\ \hline
        M1, M2, M5 & 0.454 & 0.790 \\ \hline
        M2, M3, M4 & 0.452 & 0.787 \\ \hline
        M1, M4, M5 & 0.451 & 0.771 \\ \hline
        M2, M4, M5 & 0.449 & 0.781 \\ \hline
        M2, M5 & 0.446 & 0.782 \\ \hline
        M3, M4 & 0.446 & 0.794 \\ \hline
        M1, M5 & 0.445 & 0.764 \\ \hline
        M1, M2, M4 & 0.443 & 0.780 \\ \hline
        M1, M2 & 0.440 & 0.776 \\ \hline
        M4, M5 & 0.437 & 0.770 \\ \hline
        M2, M4 & 0.428 & 0.782 \\ \hline
        M1, M4 & 0.420 & 0.776 \\ \hline
        \textbf{Average score} & \textbf{0.453} & \textbf{0.782} \\ \hline
    \end{tabular}
    \label{tab:mlt_tm_mc}
\end{table}

\begin{table}[!ht]
    \centering
    \caption{\textbf{MLT B-C performance in \textit{macroF1} scores order on CAD test}}
    \begin{tabular}{|l|l|l|}
    \hline
        \textbf{Model Combination} & \textbf{\textit{macroF1}} & \textbf{Accuracy} \\ \hline
        M2, M3, M5 & 0.715 & 0.859 \\ \hline
        M1, M2, M3, M4, M5 & 0.714 & 0.860 \\ \hline
        M1, M2, M3, M5 & 0.714 & 0.860 \\ \hline
        M2, M3, M4, M5 & 0.713 & 0.860 \\ \hline
        M1, M2, M3, M4 & 0.709 & 0.859 \\ \hline
        M1, M2, M3 & 0.709 & 0.859 \\ \hline
        M2, M4, M5 & 0.709 & 0.857 \\ \hline
        M1, M3, M4, M5 & 0.709 & 0.859 \\ \hline
        M2, M5 & 0.708 & 0.858 \\ \hline
        M2, M3 & 0.708 & 0.860 \\ \hline
        M1, M2, M4, M5 & 0.708 & 0.857 \\ \hline
        M1, M3, M5 & 0.708 & 0.859 \\ \hline
        M3, M4 & 0.706 & 0.857 \\ \hline
        M3, M4, M5 & 0.706 & 0.857 \\ \hline
        M2, M3, M4 & 0.706 & 0.859 \\ \hline
        M1, M3, M4 & 0.706 & 0.860 \\ \hline
        M1, M2, M5 & 0.705 & 0.856 \\ \hline
        M1, M3 & 0.705 & 0.860 \\ \hline
        M3, M5 & 0.704 & 0.857 \\ \hline
        M1, M4, M5 & 0.704 & 0.858 \\ \hline
        M1, M2, M4 & 0.701 & 0.854 \\ \hline
        M1, M5 & 0.699 & 0.856 \\ \hline
        M2, M4 & 0.697 & 0.849 \\ \hline
        M1, M2 & 0.697 & 0.855 \\ \hline
        M4, M5 & 0.694 & 0.850 \\ \hline
        M1, M4 & 0.688 & 0.853 \\ \hline
        \textbf{Average score} & \textbf{0.706} & \textbf{0.857} \\ \hline
    \end{tabular}
    \label{tab:mlt_bc_on_cad}
\end{table}

\begin{table}[!ht]
    \centering
    \caption{\textbf{MLT-plus-TM B-C performance in \textit{macroF1} scores order on CAD test}}
    \begin{tabular}{|l|l|l|}
    \hline
        \textbf{Model Combination} & \textbf{\textit{macroF1}} & \textbf{Accuracy} \\ \hline
        M2, M3 & 0.726 & 0.848 \\ \hline
        M1, M3, M5 & 0.724 & 0.845 \\ \hline
        M1, M2, M3, M5 & 0.723 & 0.852 \\ \hline
        M1, M3, M4, M5 & 0.723 & 0.847 \\ \hline
        M1, M2, M3, M4 & 0.723 & 0.847 \\ \hline
        M1, M2, M3 & 0.722 & 0.845 \\ \hline
        M1, M2, M3, M4, M5 & 0.722 & 0.851 \\ \hline
        M2, M3, M4 & 0.722 & 0.848 \\ \hline
        M2, M3, M5 & 0.721 & 0.847 \\ \hline
        M2, M3, M4, M5 & 0.720 & 0.847 \\ \hline
        M1, M3 & 0.720 & 0.843 \\ \hline
        M3, M4, M5 & 0.719 & 0.847 \\ \hline
        M1, M3, M4 & 0.719 & 0.846 \\ \hline
        M3, M5 & 0.717 & 0.839 \\ \hline
        M1, M2, M4, M5 & 0.717 & 0.844 \\ \hline
        M1, M2, M5 & 0.716 & 0.842 \\ \hline
        M2, M4, M5 & 0.714 & 0.842 \\ \hline
        M3, M4 & 0.713 & 0.845 \\ \hline
        M1, M4, M5 & 0.712 & 0.838 \\ \hline
        M1, M5 & 0.712 & 0.840 \\ \hline
        M2, M5 & 0.712 & 0.844 \\ \hline
        M1, M2, M4 & 0.711 & 0.833 \\ \hline
        M1, M2 & 0.708 & 0.833 \\ \hline
        M2, M4 & 0.705 & 0.837 \\ \hline
        M1, M4 & 0.701 & 0.833 \\ \hline
        M4, M5 & 0.699 & 0.831 \\ \hline
        \textbf{Average score} & \textbf{0.716} & \textbf{0.843} \\ \hline
    \end{tabular}
    \label{tab:mlt_tm_bc_on_cad}
\end{table}

\begin{table}[!ht]
    \centering
    \caption{\textbf{MLT B-C performance in \textit{macroF1} scores order on dataset~\cite{rottger2020hatecheck}}}
    \begin{tabular}{|l|l|l|}
    \hline
        \textbf{Model Combination} & \textbf{\textit{macroF1}} & \textbf{Accuracy} \\ \hline
        M1, M2, M3 & 0.697 & 0.723 \\ \hline
        M1, M2, M3, M4 & 0.693 & 0.718 \\ \hline
        M1, M2, M3, M5 & 0.690 & 0.720 \\ \hline
        M1, M2, M3, M4, M5 & 0.688 & 0.717 \\ \hline
        M2, M3 & 0.688 & 0.722 \\ \hline
        M2, M3, M5 & 0.688 & 0.726 \\ \hline
        M3, M4 & 0.683 & 0.710 \\ \hline
        M2, M3, M4 & 0.681 & 0.712 \\ \hline
        M1, M3, M5 & 0.681 & 0.702 \\ \hline
        M1, M3, M4, M5 & 0.681 & 0.703 \\ \hline
        M2, M3, M4, M5 & 0.681 & 0.713 \\ \hline
        M3, M5 & 0.681 & 0.712 \\ \hline
        M3, M4, M5 & 0.673 & 0.703 \\ \hline
        M1, M2, M5 & 0.661 & 0.684 \\ \hline
        M1, M2, M4, M5 & 0.657 & 0.678 \\ \hline
        M2, M5 & 0.646 & 0.674 \\ \hline
        M1, M3 & 0.644 & 0.654 \\ \hline
        M1, M3, M4 & 0.642 & 0.651 \\ \hline
        M2, M4, M5 & 0.641 & 0.666 \\ \hline
        M2, M4 & 0.635 & 0.659 \\ \hline
        M1, M2, M4 & 0.626 & 0.635 \\ \hline
        M1, M2 & 0.615 & 0.622 \\ \hline
        M1, M4, M5 & 0.586 & 0.590 \\ \hline
        M1, M5 & 0.585 & 0.589 \\ \hline
        M4, M5 & 0.582 & 0.594 \\ \hline
        M1, M4 & 0.515 & 0.516 \\ \hline
        \textbf{Average score} & \textbf{0.652} & \textbf{0.673} \\ \hline
    \end{tabular}
    \label{tab:mlt_bc_on_hatecheck}
\end{table}

\begin{table}[!ht]
    \centering
    \caption{\textbf{MLT-plus-TM B-C performance in \textit{macroF1} scores order on dataset~\cite{rottger2020hatecheck}}}
    \begin{tabular}{|l|l|l|}
    \hline
        \textbf{Model Combination} & \textbf{\textit{macroF1}} & \textbf{Accuracy} \\ \hline
        M1, M3 & 0.709 & 0.763 \\ \hline
        M1, M3, M4 & 0.699 & 0.745 \\ \hline
        M1, M2, M3, M4 & 0.698 & 0.753 \\ \hline
        M1, M2, M3 & 0.697 & 0.757 \\ \hline
        M1, M2, M3, M5 & 0.692 & 0.745 \\ \hline
        M1, M2, M3, M4, M5 & 0.690 & 0.741 \\ \hline
        M3, M4 & 0.689 & 0.743 \\ \hline
        M2, M3 & 0.687 & 0.754 \\ \hline
        M2, M3, M4 & 0.686 & 0.747 \\ \hline
        M1, M3, M4, M5 & 0.685 & 0.741 \\ \hline
        M1, M3, M5 & 0.681 & 0.745 \\ \hline
        M1, M2 & 0.678 & 0.729 \\ \hline
        M1, M2, M4 & 0.676 & 0.727 \\ \hline
        M2, M3, M5 & 0.675 & 0.743 \\ \hline
        M2, M3, M4, M5 & 0.674 & 0.740 \\ \hline
        M3, M4, M5 & 0.670 & 0.731 \\ \hline
        M1, M2, M4, M5 & 0.668 & 0.715 \\ \hline
        M1, M2, M5 & 0.667 & 0.719 \\ \hline
        M3, M5 & 0.658 & 0.737 \\ \hline
        M1, M5 & 0.656 & 0.694 \\ \hline
        M2, M4 & 0.656 & 0.701 \\ \hline
        M1, M4, M5 & 0.655 & 0.695 \\ \hline
        M2, M4, M5 & 0.652 & 0.703 \\ \hline
        M2, M5 & 0.651 & 0.704 \\ \hline
        M4, M5 & 0.620 & 0.663 \\ \hline
        M1, M4 & 0.605 & 0.617 \\ \hline
        \textbf{Average score} & \textbf{0.672} & \textbf{0.725} \\ \hline
    \end{tabular}
    \label{tab:mlt_tm_bc_on_hatecheck}
\end{table}

\subsection{Inference from Combiner Performance}\label{sec:combinerPerformance}

Section~\ref{sec:pmc} and ~\ref{sec:pbi} demonstrate that MLT-plus-TM significantly improves performance, surpassing all individual models and achieving best overall performance. 
Even when we transfer the learned MLT combining weights and threshold cut-off values from CAD to the dataset~\cite{rottger2020hatecheck}, where the data distribution differs substantially from the training data, the final MLT-plus-TM remains at least on par with the best participating model.

We also provide MLT and MLT-plus-TM results for all possible model combinations (~\ref{tab:mlt_mc}, ~\ref{tab:mlt_tm_mc}, ~\ref{tab:mlt_bc_on_cad}, ~\ref{tab:mlt_tm_bc_on_cad}, ~\ref{tab:mlt_bc_on_hatecheck}, and ~\ref{tab:mlt_tm_bc_on_hatecheck}). The best combiner for ~\acrshort{m-c} on the CAD test is (M2, M3, M5) with a \textit{macroF1} score of $0.469$, compared to $0.467$ in all combinations of (M1, M2, M3, M4, M5). In ~\acrshort{b-c} on the CAD test, the best combination is (M2, M3) with a \textit{macroF1} score of $0.726$, while all combinations yield $0.722$. For ~\acrshort{b-c} on dataset~\cite{rottger2020hatecheck} transfer learning, the best combination is (M1, M3) with a \textit{macroF1} score of $0.709$, whereas all combinations yield $0.690$. These results clearly demonstrate that all combinations fall short of achieving the highest possible performance.
Nevertheless, we still report the performance by combining all five models (M1, M2, M3, M4, M5).

\begin{figure}[!th]
    \centering
    \caption{\textbf{Combiner Performance F1 scores on dataset~\cite{rottger2020hatecheck} using transfer learning:}  MLT \textbf{B-C} performance and MLT-plus-TM \textbf{B-C} performance for all the possible combinations. MLT \textbf{B-C} performance is taken from table~\ref{tab:mlt_bc_on_hatecheck}, and corresponding MLT-plus-TM \textbf{B-C} performance is taken from table~\ref{tab:mlt_tm_bc_on_hatecheck}.}
    \label{fig:hatecheck_f1_mlt_mlt+tm}
    {
    \includegraphics[width=0.9\textwidth]{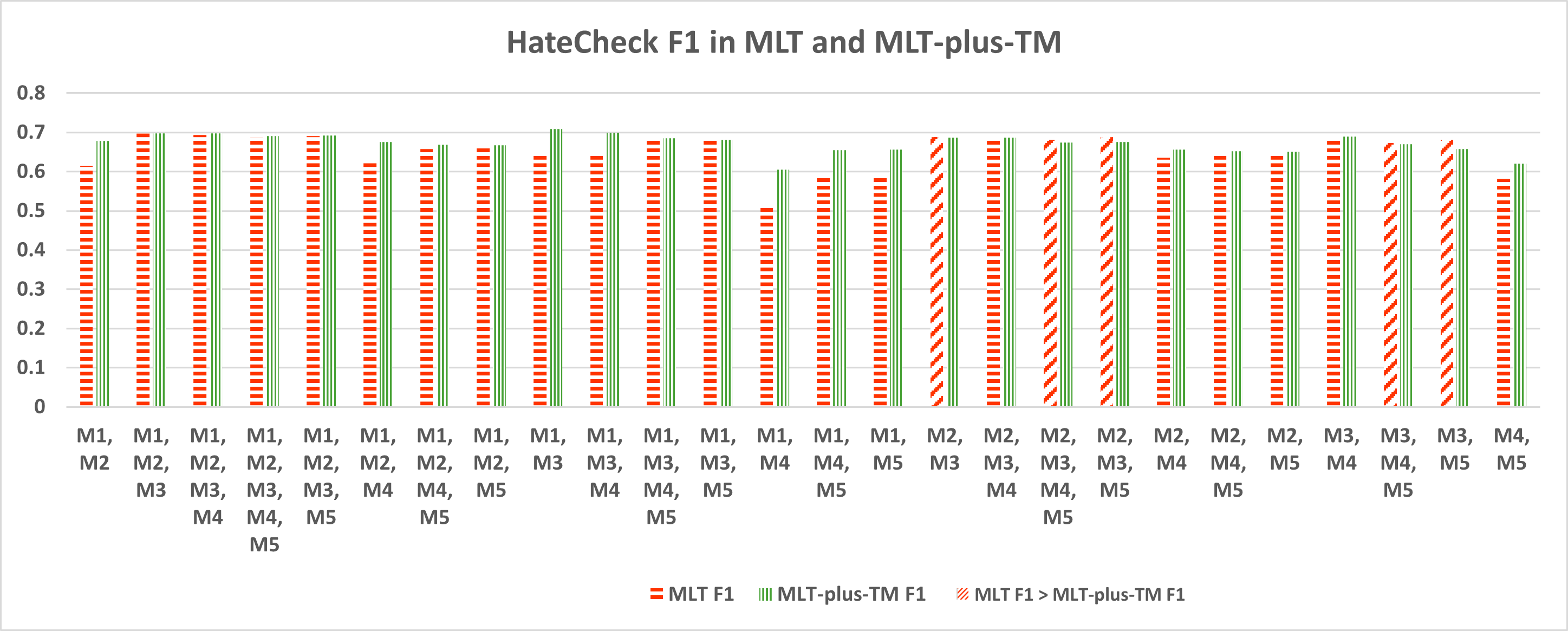} 
    } 
\end{figure}

\begin{figure}[!th]
    \centering
    \caption{\textbf{Combiner Performance Accuracy scores on dataset~\cite{rottger2020hatecheck} using transfer learning:}  MLT \textbf{B-C} performance and MLT-plus-TM \textbf{B-C} performance for all the possible combinations. MLT \textbf{B-C} performance is taken from Table~\ref{tab:mlt_bc_on_hatecheck}, and corresponding MLT-plus-TM \textbf{B-C} performance is taken from Table~\ref{tab:mlt_tm_bc_on_hatecheck}.}
    \label{fig:hatecheck_acc_mlt_mlt+tm}
    {
    \includegraphics[width=0.9\textwidth]{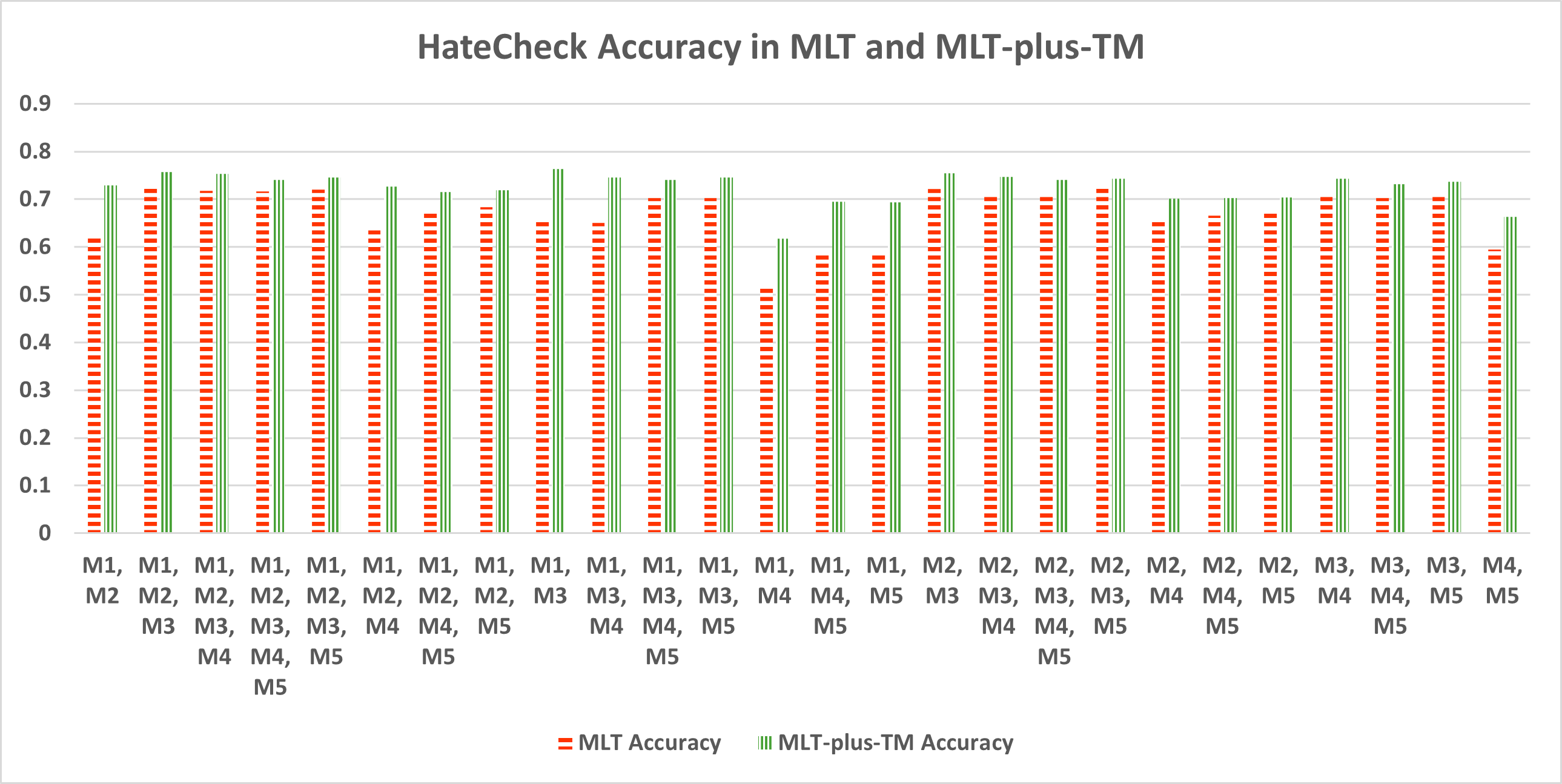} 
    } 
\end{figure}

\ifbool{false}
{
From our experiments, we know that the top three performing models are Bertweet(M3), Bert(M2) and XLNet (M5). From Tables~\ref{tab:mlt_mc}, ~\ref{tab:mlt_tm_mc}, ~\ref{tab:mlt_bc_on_cad}, and ~\ref{tab:mlt_tm_bc_on_cad}, we can see that
the top performing combiner across all these tables are some combination of these three models. For example, in case of CAD {\fontfamily{qcr}\selectfont test} data in ~\acrshort{b-c}: the best F1 score is achieved by the combination of M2, M3 and M5 (table~\ref{tab:mlt_bc_on_cad}) while applying TM on top of MLT, the best F1 score is achieved by using the combination of M2 and M3 (table~\ref{tab:mlt_tm_bc_on_cad}) respectively. In transfer learning on out-of-distribution data (Table~\ref{tab:mlt_bc_on_hatecheck} and~\ref{tab:mlt_tm_bc_on_hatecheck}), the dynamics can change due to datasets having different data distribution. However, combining models overall has benefits in terms of improving the metrics. TM improves the F1-score in both the
CAD {\fontfamily{qcr}\selectfont test} dataset as well as the dataset~\cite{rottger2020hatecheck} for all model combinations. This can be verified by comparing Table~\ref{tab:mlt_bc_on_cad} with Table~\ref{tab:mlt_tm_bc_on_cad}, and comparing Table~\ref{tab:mlt_bc_on_hatecheck} with Table~\ref{tab:mlt_tm_bc_on_hatecheck}, for the corresponding model combinations. This is shown graphically in Figure~\ref{fig:cad_test_tm} and Figure~\ref{fig:hatecheck_tm}.}



\ifbool{false}
{
Further, for each combination of respective size, there is at least one model as the participant from the top three ones i.e., M2, M3 or M5:
\begin{itemize}
    \item Combination size = 2:\\
    There are a total of $10$ combinations possible with combination size set as 2. In case CAD {\fontfamily{qcr}\selectfont test} data, in M-C, we see that the best combination is of M2 and M3 with F1 score as $0.4062$  and accuracy as $0.8165$ (table~\ref{tab:mlt_mc}); if we see MLT-plus-TM performance then again M2 and M3 emerges as the best performing combination of size two with F1 score as 0.4590 (table~\ref{tab:mlt_tm_mc}). In case of B-C, we can see from table~\ref{tab:mlt_bc_on_cad}, the best performing model combination is M2 and M5 with F1 score as $0.708$ and accuracy as $0.858$. In case of MLT-plus-TM, M2 and M3 (table~\ref{tab:mlt_tm_bc_on_cad}) pop up as the best combiner with F1 score as $0.726$ and accuracy as $0.848$.\\
    In the case of the dataset~\cite{rottger2020hatecheck}, we can see that the combination of M2 and M3 emerges on the top with $0.688$ as the F1 score and $0.722$ as the accuracy (table~\ref{tab:mlt_bc_on_hatecheck}). In case of MLT-plus-TM, we see that the combination of M1 and M3 is the best performing overall combination with F1 score as $0.709$ and accuracy as $0.763$ (table~\ref{tab:mlt_tm_bc_on_hatecheck}).
    \item Combination size = 3:\\
    In this combination size, there are also a total of $10$ possible combinations. In case CAD {\fontfamily{qcr}\selectfont test} data, in M-C, M1, M2 and M3 emerges as winner with F1 score as $0.4308$ (table~\ref{tab:mlt_mc}); if we see the performance of MLT-plus-TM then M2, M3 and M5 emerges as the best performing combination of size three and in fact takes up the top spot of table~\ref{tab:mlt_tm_mc} with F1 score as 0.4687. If case of B-C performance, M2, M3 and M5 emerges as the best performing combination for the CAD {\fontfamily{qcr}\selectfont test} data. The best F1 score achieved is $0.715$ while the best accuracy is $0.859$. It is intuitive that the combination size equal to the top performing models would yield the best metrics.
    In case of threshold moving technique, M1, M3 and M5 emerges as the best performing combination with F1 score as $0.724$ and accuracy as $0.845$ (ref table~\ref{tab:mlt_bc_on_cad} and \ref{tab:mlt_tm_bc_on_cad}).\\
    M1, M2 and M3 is the best performing combination for this combination size 3, with F1 score as $0.697$ and accuracy as $0.723$. Upon applying threshold moving technique, M1, M3 and M4 emerges as the best performing model with F1 score as $0.699$ and accuracy as $0.745$. M4 is one of the participant of the best performing combination (combination size 3) in threshold moving technique; one plausible reason could be, good contribution from M4 when the combination size is 3 (refer table~\ref{tab:mlt_bc_on_hatecheck} and \ref{tab:mlt_tm_bc_on_hatecheck}). 
    \item Combination size = 4: \\
    There are a total of $5$ combinations possible when the combination size equals 4. For M-C performance, in CAD {\fontfamily{qcr}\selectfont test} data, we see that M1, M2, M3 and M5 comes out as winner with F1 score as $0.4375$ and accuracy as $0.8213$ (table~\ref{tab:mlt_mc}); for MLT-plus-TM, we M1, M3, M4 and M5 as the best performing combination of size four table~\ref{tab:mlt_tm_mc} with F1 score as 0.4654. If case of B-C performance, M1, M2, M3 and M5 emerges as the best combination with F1 score as $0.714$ and accuracy as $0.86$ for the CAD {\fontfamily{qcr}\selectfont test} data. By using the threshold moving technique, the same combination again pops up as the winner with marginal gain in F1 score $(0.723)$ and accuracy as $0.852$ (refer table~\ref{tab:mlt_bc_on_cad} and \ref{tab:mlt_tm_bc_on_cad}).\\
    In the case of the dataset~\cite{rottger2020hatecheck}, M1, M2, M3 and M4 emerges as the best performing combination (combination size = 4), with F1 score as $0.693$ and accuracy as $0.718$. When using the threshold moving technique, again M1, M2, M3 and M4 emerges as the best performing combination with $0.698$ as the F1 score and $0.753$ as the accuracy (refer table~\ref{tab:mlt_bc_on_hatecheck} and \ref{tab:mlt_tm_bc_on_hatecheck}).
    \item Combination size = 5: \\
    In case CAD {\fontfamily{qcr}\selectfont test} data, the combination size of five emerges as the best overall performer with F1 score as $0.4392$ and accuracy as $0.8205$ (table~\ref{tab:mlt_mc}); if we see the performance of MLT-plus-TM then we can see that the combination is second best with F1 score as 0.467. If case of B-C performance on CAD {\fontfamily{qcr}\selectfont test} data, the F1 score is $0.714$ and the accuracy is $0.86$. Upon applying threshold moving technique, the F1 score becomes $0.722$ and the accuracy becomes $0.851$.\\
    In the case of the dataset~\cite{rottger2020hatecheck}, F1 score is $0.688$ and accuracy is $0.717$; while upon using the threshold moving technique, the F1 score becomes $0.69$ and the accuracy becomes $0.741$ respectively.
\end{itemize}
}{}

\subsubsection{MLT performance analysis}
\label{sec:mlt-analysis}
The results(~\ref{tab:cad_mc}, ~\ref{tab:bc_on_cad}) indicate that MLT achieves the best accuracy, surpassing the best individual model, BERTweet ($0.821$ vs. $0.806$ in ~\acrshort{m-c} and $0.86$ vs. $0.85$ in ~\acrshort{b-c}). 
A plausible explanation for this is that CAD is highly imbalanced, and MLT may further enhance the advantage of the majority class, resulting in higher accuracy. However, when considering \textit{macroF1}, which treats all classes equally and places greater emphasis on the minority class, MLT may not necessarily lead to a higher \textit{macroF1} score. Although MLT may not directly boost the \textit{macroF1} score, 
 the higher accuracy achieved through MLT leaves more room for improvement in the \textit{macroF1} score through TM. Consequently, MLT-plus-TM achieves the best performance, surpassing that of Bertweet ($0.467$ vs. $0.441$ in ~\acrshort{m-c} and $0.722$ vs $0.711$ in ~\acrshort{b-c}) in our targeted metric, macro-averaged F1 (\textit{macroF1}).  With MLT-plus-TM, both 
~\acrshort{m-c} and ~\acrshort{b-c} on CAD achieve the best performance in \textit{macroF1}.

The goal of 
MLT is to train optimized combining weights to achieve better performance than that of the individual models on a given dataset.
However, when such weights are applied to a different dataset, such as in a transfer learning scenario, the dynamics of the combination effect may shift. In Table~\ref{tab:bc_on_hatecheck} on dataset~\cite{rottger2020hatecheck}, MLT slightly underperforms when compared with the best individual model, Bertweet ($0.717$ vs. $0.739$ in accuracy and $0.688$ vs. $0.692$ in \textit{macroF1}). This indicates that the combining weights trained on CAD are not optimal for dataset~\cite{rottger2020hatecheck}. 

When we examining individual model performance in Table~\ref{tab:bc_on_cad} and ~\ref{tab:bc_on_hatecheck}, we notice that the inference on CAD test and dataset~\cite{rottger2020hatecheck} does not always follow the same pattern. Although BERTweet (M3) has the best performance and Bloom (M5) is the weakest in both datasets, in Table~\ref{tab:bc_on_cad} on CAD test, BERTweet (M3), Bigbird (M1), BERT (M2) and XLNet (M5) are very close in performance. However, in Table~\ref{tab:bc_on_hatecheck} on dataset~\cite{rottger2020hatecheck}, BertWeet stands out from the rest. Following BERTweet are BERT and XLNet. Bigbird performs significantly worse than the three better models, and Bloom is at the bottom. The performance gap among individual models on dataset~\cite{rottger2020hatecheck} is much wider than on CAD. Bigbird and Bloom show the sharpest performance drop on dataset~\cite{rottger2020hatecheck}. The difference in the pattern of relative model performance between CAD and dataset~\cite{rottger2020hatecheck} shows that the optimal combining weights for CAD may not be optimal on dataset~\cite{rottger2020hatecheck}, which is not surprising in transfer learning. 

In MLT-plus-TM, we use a separate mechanism, TM, to boost final performance in \textit{macroF1} after MLT.

\subsubsection{TM performance analysis}
\label{sec:tm-analysis}

As shown in Figure~\ref{fig:cad_hc_bc_dataset}, the radio between "Non-hateFul" to "Hateful" is $80\,\%$ to $20\,\%$ on CAD.
We train threshold cut-off values to detect more instances of minority class, "Hateful", in order to maximize the overall \textit{macroF1} score. It's important to note that the trained cut-off value is lower than the default 0.5 threshold. This adjustment increases the number of true positives (more instances of "Hateful" instances are correctly classified as such), but it can also decrease the number of true negatives (instances of "Non-hateful" are correctly classified as such). Consequently, the effect of TM effect on accuracy depends on the combined count of true positives and true negatives. If this count increases, accuracy increases; otherwise, accuracy decreases. Considering that $80\,\%$ of data are labeled as "Non-hateful" in the CAD dataset, the decrease in true negatives may be much larger than the increase in true negatives. As a result, the overall accuracy after applying TM may decrease in CAD.

This effect is observed in both ~\acrshort{m-c} and ~\acrshort{b-c} on CAD. In ~\acrshort{m-c} (Table~\ref{tab:mlt_mc}, ~\ref{tab:mlt_tm_mc}), the \textit{macroF1} score for all combiners increases from $0.410$ to $0.458$, but the averaged accuracy decreases from $0.817$ to $0.782$. In ~\acrshort{b-c} (Table~\ref{tab:mlt_bc_on_cad}, ~\ref{tab:mlt_tm_bc_on_cad}), the \textit{macroF1} score increases from $0.706$ to $0.716$, while the averaged accuracy decreases from $0.875$ to $0.843$. This trade-off is acceptable because \textit{macroF1} is a more important metric as in this highly imbalanced data situation.

The effect of applying TM in a dataset with an opposite class distribution compared to the training data is more complex.

In dataset~\cite{rottger2020hatecheck}, when we apply threshold cut-off value trained from CAD on the 'Hateful' class, which is lower than the 0.5 cut-off, it increases the chance of detecting "Hateful" instances more easily. Consequently, more "Hateful" samples are detected, resulting in an increase in the number of true positives for the "Hateful" class. While the number of false negatives for the "Non-hateful" class may also increase, and the number of true negatives decreases, since the majority of the data is "Hateful" ($68\,\%$), the increase in true positives in "Hateful" is likely to be larger than the decrease in true negatives in "Non-hateful". As a result, the overall accuracy is more likely to improve. This is opposite to what happens in CAD and is not unexpected. Figure~\ref{fig:hatecheck_acc_mlt_mlt+tm} illustrates this effect: in all 26 combiners, MLT-plus-TM achieves higher accuracy scores than MLT.

The case of \textit{macroF1} is more delicate. The precision for the "Hateful" class may increase, while the precision for the "Non-hateful" class may decrease. The recall for both classes may also be affected. The \textit{macroF1} score is the average of the F1 scores for both classes, so it will depend on how these changes in precision and recall balance out.   Figure~\ref{fig:hatecheck_f1_mlt_mlt+tm} shows that in all 26 combiners, 21 combiners achieve higher \textit{macroF1} scores after TM. Only 5 (represented by the upward diagonal strips pattern bars)  produce slightly lower \textit{macroF1} scores.

Overall, the threshold cut-off values trained in CAD work very well on dataset~\cite{rottger2020hatecheck},  despite the opposite class composition of the two datasets. From Table~\ref{tab:mlt_bc_on_hatecheck} and ~\ref{tab:mlt_tm_bc_on_hatecheck}, both \textit{macroF1} and accuracy scores are improved after applying TM using the CAD trained threshold values:  average gain in \textit{macroF1} for the 26 combiners is from $0.652$ to $0.672$, and the average gain in accuracy is from $0.673$ to $0.725$. 
The reason why TM works well even when applied
to a dataset with an opposite class composition may be that the models trained on CAD have an implicit bias toward the
majority class, and TM helps to offset this bias. When the offset
brings the predictions closer to the ground truth, it results in performance gains. Compared to MLT, MLT-plus-TM improves both accuracy ($0.741$ vs $0.717$) and \textit{macroF1} ($0.69$ vs $0.688$) in out-of-distribution transfer learning on dataset~\cite{rottger2020hatecheck} (~\ref{tab:bc_on_hatecheck}).

\subsubsection{Reinforcing the Weak Model}
The combiner is not necessarily composed of the best participating models. We try to develop an intuition about how superior and inferior models can combine together and obtain a better metric overall. First, we look at the overall metric results, then, we consider how the weights of the combiner adjust themselves among the superior and the inferior models.

Bloom  (M4) has the weakest performance among the five participating models. When M4 is compared with M1, M2,
M3 or M5, its \textit{macroF1} score and accuracy results are inferior for both the CAD {\fontfamily{qcr}\selectfont test} and dataset~\cite{rottger2020hatecheck}. Though the  results of M4 are inferior, it may still be possible that the hyperspace of M4 might carry some useful information that could support other models in establishing the correct inference. This is where MLT plays a big role, the combinations gives M4 some opportunity to shine in the company of other superior models. Similarly, the case with M1 and M5 when compared to M2 and M3. Further, when M4 is combined with any other model, there is a considerable increment in both the metrics (\textit{macroF1} score and accuracy), compared to M4 alone; while when compared with the individual metrics of the participating models, there is always a small increment. One of the major 
reasons for this observation is the higher weight of the other participating models compared to M4.

If we analyze the weights of the participating models in the combiner, we find that the models that are inferior to their participating peers, tend to have a lower weight. For example, in the case of a combiner with models M1 and M4, the weight of model M1 is $1.246$, while the weight of model M4 is $0.885$. For a three model combiner with models M1, M2 and M4, the weight of model M1 is $0.692$, model M2 is $0.755$, and model M4 is $0.421$, respectively. In the case of a combiner with four models (M1, M2, M3 and M4), we have the weights as $0.566$, $0.615$, $0.697$ and $0.341$ for models M1, M2, M3 and M4, respectively. Similar weight distribution is seen when all the models participate for the combiner, i.e., $0.327$, $0.367$, $0.496$, $0.225$ and $0.419$ for M1, M2, M3, M4 and M5, respectively. Further, from these examples one can see that the weights tend to reflect the individual strength of the models. The inferences and the experimental results are in agreement with \textit{Vassilev et.al.}~\cite{ssnet}, where the authors show that in a combination of superior and inferior model, the inferior model becomes a supporting agent to the combiner, resulting into overall improvements of metrics when compared against the metrics of individual models. Overall, combining more models results in better performance.

\begin{table}[!ht]
    \centering
    \caption{\textbf{B-C Accuracy scores on dataset~\cite{rottger2020hatecheck} 29 functionalities}}
    \resizebox{\textwidth}{!}{%
    \begin{tabular}{|l|l|l|l|l|l|l|l|l|l|l|l|l|l|l|l|}
    \hline
        \textbf{} & \textbf{Functionality} & \textbf{B-D} & \textbf{B-F} & \textbf{Google-P} & \textbf{BIGBIRD} & \textbf{BERT} & \textbf{BERTWEET} & \textbf{BLOOM} & \textbf{XLNET} & \textbf{MLT\_1} & \textbf{MLT-plus-TM\_1} & \textbf{MLT\_2} & \textbf{MLT-plus-TM\_2} & \textbf{MLT\_3} & \textbf{MLT-plus-TM\_3} \\ \hline
        F1 & derog\_neg\_emote\_h & 0.886 & 0.907 & 0.986 & 0.257 & 0.657 & 0.729 & 0.086 & 0.579 & 0.621 & 0.779 & 0.679 & 0.814 & 0.600 & 0.893 \\ \hline
        F2 & derog\_neg\_attrib\_h & 0.886 & 0.843 & 0.957 & 0.436 & 0.864 & 0.921 & 0.107 & 0.764 & 0.871 & 0.964 & 0.907 & 0.964 & 0.850 & 0.993 \\ \hline
        F3 & derog\_dehum\_h & 0.914 & 0.807 & 0.986 & 0.414 & 0.914 & 0.921 & 0.029 & 0.814 & 0.886 & 0.921 & 0.900 & 0.936 & 0.886 & 0.971 \\ \hline
        F4 & derog\_impl\_h & 0.714 & 0.614 & 0.850 & 0.207 & 0.443 & 0.507 & 0.100 & 0.586 & 0.471 & 0.614 & 0.493 & 0.643 & 0.436 & 0.729 \\ \hline
        F5 & threat\_dir\_h & 0.872 & 0.759 & 1.000 & 0.474 & 0.767 & 0.669 & 0.023 & 0.812 & 0.759 & 0.850 & 0.789 & 0.910 & 0.737 & 0.947 \\ \hline
        F6 & threat\_norm\_h & 0.914 & 0.836 & 1.000 & 0.357 & 0.807 & 0.786 & 0.079 & 0.664 & 0.786 & 0.879 & 0.821 & 0.914 & 0.757 & 0.943 \\ \hline
        F7 & slur\_h & 0.604 & 0.410 & 0.660 & 0.722 & 0.861 & 0.938 & 0.000 & 0.771 & 0.924 & 0.958 & 0.931 & 0.972 & 0.910 & 1.000 \\ \hline
        \rowcolor{lightgray} F8 & slur\_homonym\_nh & 0.667 & 0.700 & 0.633 & 0.767 & 0.600 & 0.533 & 0.800 & 0.633 & 0.633 & 0.533 & 0.600 & 0.467 & 0.667 & 0.467 \\ \hline
        \rowcolor{lightgray} F9 & slur\_reclaimed\_nh & 0.395 & 0.333 & 0.284 & 0.383 & 0.358 & 0.506 & 0.667 & 0.099 & 0.383 & 0.259 & 0.370 & 0.272 & 0.370 & 0.160 \\ \hline
        F10 & profanity\_h & 0.829 & 0.729 & 1.000 & 0.593 & 0.914 & 0.886 & 0.229 & 0.821 & 0.871 & 0.921 & 0.879 & 0.929 & 0.857 & 0.964 \\ \hline
        \rowcolor{lightgray} F11 & profanity\_nh & 0.990 & 1.000 & 0.980 & 0.950 & 0.920 & 0.950 & 0.990 & 0.850 & 0.930 & 0.900 & 0.930 & 0.900 & 0.930 & 0.880 \\ \hline
        F12 & ref\_subs\_clause\_h & 0.871 & 0.807 & 0.993 & 0.379 & 0.821 & 0.857 & 0.471 & 0.843 & 0.807 & 0.914 & 0.836 & 0.921 & 0.779 & 0.957 \\ \hline
        F13 & ref\_subs\_sent\_h & 0.857 & 0.707 & 1.000 & 0.474 & 0.857 & 0.782 & 0.534 & 0.722 & 0.797 & 0.880 & 0.812 & 0.872 & 0.789 & 0.925 \\ \hline
        F14 & negate\_pos\_h & 0.850 & 0.607 & 0.964 & 0.286 & 0.686 & 0.621 & 0.057 & 0.593 & 0.571 & 0.771 & 0.621 & 0.757 & 0.564 & 0.836 \\ \hline
        \rowcolor{lightgray} F15 & negate\_neg\_nh & 0.128 & 0.120 & 0.038 & 0.789 & 0.421 & 0.511 & 0.940 & 0.549 & 0.556 & 0.338 & 0.534 & 0.346 & 0.586 & 0.211 \\ \hline
        F16 & phrase\_question\_h & 0.807 & 0.750 & 0.993 & 0.500 & 0.614 & 0.714 & 0.000 & 0.500 & 0.664 & 0.793 & 0.693 & 0.764 & 0.636 & 0.836 \\ \hline
        F17 & phrase\_opinion\_h & 0.857 & 0.759 & 0.985 & 0.466 & 0.827 & 0.805 & 0.045 & 0.669 & 0.759 & 0.887 & 0.797 & 0.910 & 0.737 & 0.940 \\ \hline
        \rowcolor{lightgray} F18 & ident\_neutral\_nh & 0.206 & 0.587 & 0.841 & 0.929 & 0.802 & 0.897 & 1.000 & 0.849 & 0.929 & 0.778 & 0.897 & 0.770 & 0.952 & 0.651 \\ \hline
        \rowcolor{lightgray} F19 & ident\_pos\_nh & 0.217 & 0.529 & 0.540 & 0.937 & 0.651 & 0.608 & 0.963 & 0.455 & 0.704 & 0.534 & 0.683 & 0.492 & 0.725 & 0.360 \\ \hline
        \rowcolor{lightgray} F20 & counter\_quote\_nh & 0.266 & 0.329 & 0.156 & 0.803 & 0.491 & 0.376 & 0.694 & 0.647 & 0.601 & 0.405 & 0.566 & 0.376 & 0.624 & 0.254 \\ \hline
        \rowcolor{lightgray} F21 & counter\_ref\_nh & 0.291 & 0.298 & 0.184 & 0.730 & 0.475 & 0.560 & 0.745 & 0.617 & 0.589 & 0.404 & 0.589 & 0.390 & 0.610 & 0.284 \\ \hline
        \rowcolor{lightgray} F22 & target\_obj\_nh & 0.877 & 0.846 & 0.954 & 1.000 & 1.000 & 1.000 & 1.000 & 1.000 & 1.000 & 1.000 & 1.000 & 1.000 & 1.000 & 1.000 \\ \hline
        \rowcolor{lightgray} F23 & target\_indiv\_nh & 0.277 & 0.554 & 0.846 & 0.846 & 0.600 & 0.646 & 1.000 & 0.646 & 0.692 & 0.662 & 0.662 & 0.615 & 0.692 & 0.569 \\ \hline
        \rowcolor{lightgray} F24 & target\_group\_nh & 0.355 & 0.597 & 0.629 & 0.839 & 0.532 & 0.613 & 0.919 & 0.694 & 0.710 & 0.597 & 0.645 & 0.548 & 0.710 & 0.452 \\ \hline
        F25 & spell\_char\_swap\_h & 0.699 & 0.586 & 0.887 & 0.571 & 0.609 & 0.556 & 0.105 & 0.541 & 0.669 & 0.744 & 0.677 & 0.782 & 0.624 & 0.842 \\ \hline
        F26 & spell\_char\_del\_h & 0.593 & 0.479 & 0.743 & 0.557 & 0.671 & 0.707 & 0.079 & 0.514 & 0.700 & 0.821 & 0.693 & 0.836 & 0.679 & 0.900 \\ \hline
        F27 & spell\_space\_del\_h & 0.681 & 0.511 & 0.801 & 0.383 & 0.674 & 0.830 & 0.071 & 0.652 & 0.738 & 0.872 & 0.759 & 0.872 & 0.723 & 0.929 \\ \hline
        F28 & spell\_space\_add\_h & 0.439 & 0.376 & 0.740 & 0.110 & 0.266 & 0.376 & 0.121 & 0.249 & 0.243 & 0.335 & 0.272 & 0.376 & 0.231 & 0.468 \\ \hline
        F29 & spell\_leet\_h & 0.480 & 0.439 & 0.682 & 0.607 & 0.561 & 0.705 & 0.046 & 0.526 & 0.647 & 0.786 & 0.665 & 0.786 & 0.607 & 0.867 \\ \hline
    \end{tabular}
    }
     \label{tab:bc_29_funcs_hatecheck}
\end{table}

\subsection{Performance in 29 functionalities} \label{sec:ind_fun}
The test suite~\cite{rottger2020hatecheck} provides $29$ functional tests, as shown in Table ~\ref{tab:bc_29_funcs_hatecheck}. We modified original code to include our performance numbers for comparison. In MLT and MLT-plus-TM, we include three runs to provide insights into more general performance.
Each test evaluates a specific functionality and is associated with a gold standard
label ("Hateful" or "Non-hateful"). 
The table also provides performance across functional tests for the models mentioned above. 
We present our results using the same tests in Table~\ref{tab:bc_29_funcs_hatecheck}. The
`Non-hateful' functionality lines are highlighted in gray.
As expected, MLT-plus-TM improves performance in all "Hateful" functionalities compared to MLT alone, while it has the opposite effect on all "Non-hateful" functionalities.
Bloom is strongly biased towards classifying all cases as "Non-hateful", resulting in high accuracy on "Non-hateful" cases but misclassifying most "Hateful" cases.


\subsection{Performance on target groups}
We also present our models' performance on target groups, alongside the results of models reported in~\cite{rottger2020hatecheck}, as displayed in Table~\ref{tab:target_acc_hatecheck} using the modified original code.
We adhere to the Centers for Disease Control and Prevention (CDC) guidelines for \hyperlink{https://www.cdc.gov/healthcommunication/Preferred_Terms.html}{Preferred Terms for Select Population Groups \& Communities, available at https://www.cdc.gov/healthcommunication/Preferred\_Terms.html}, when referring to some target group names.
Similar to section ~\ref{sec:ind_fun}, we also include three runs for MLT and MLT-plus-TM.
MLT-plus-TM demonstrates a more balanced accuracy spread across all seven
target groups, thereby reducing performance biases among them and closing the gap with Google-P's performance. 

\begin{table}[!ht]
    \begin{center}
    \caption{\textbf{Model accuracy on test cases by target group}}
    \resizebox{\textwidth}{!}{%
    \begin{tabular}{|l|l|l|l|l|l|l|l|l|l|l|l|l|l|l|}
    \hline
        \textbf{Target group} & \textbf{B-D} & \textbf{B-F} & \textbf{Google-P} & \textbf{BIGBIRD} & \textbf{BERT} & \textbf{BERTWEET} & \textbf{BLOOM} & \textbf{XLNET} & \textbf{MLT\_1} & \textbf{MLT-plus-TM\_1} & \textbf{MLT\_2} & \textbf{MLT-plus-TM\_2} & \textbf{MLT\_3} & \textbf{MLT-plus-TM\_3} \\ \hline
        Group 1 & 0.349 & 0.523 & 0.805 & 0.634 & 0.715 & 0.710 & 0.287 & 0.713 & 0.767 & 0.781 & 0.755 & 0.770 & 0.748 & 0.796 \\ \hline
        Group 2 & 0.691 & 0.694 & 0.808 & 0.644 & 0.677 & 0.808 & 0.323 & 0.610 & 0.748 & 0.789 & 0.760 & 0.760 & 0.732 & 0.767 \\ \hline
        Group 3 & 0.739 & 0.743 & 0.808 & 0.646 & 0.772 & 0.786 & 0.373 & 0.724 & 0.798 & 0.777 & 0.789 & 0.767 & 0.800 & 0.767 \\ \hline
        Group 4 & 0.698 & 0.722 & 0.805 & 0.458 & 0.720 & 0.689 & 0.337 & 0.613 & 0.694 & 0.755 & 0.734 & 0.774 & 0.672 & 0.791 \\ \hline
        Group 5 & 0.710 & 0.371 & 0.798 & 0.418 & 0.656 & 0.477 & 0.233 & 0.496 & 0.523 & 0.634 & 0.568 & 0.696 & 0.506 & 0.736 \\ \hline
        Group 6 & 0.722 & 0.736 & 0.796 & 0.354 & 0.755 & 0.784 & 0.323 & 0.762 & 0.803 & 0.784 & 0.791 & 0.786 & 0.796 & 0.767 \\ \hline
        Group 7 & 0.705 & 0.589 & 0.805 & 0.447 & 0.515 & 0.698 & 0.235 & 0.553 & 0.591 & 0.686 & 0.637 & 0.713 & 0.589 & 0.739 \\ \hline
    \end{tabular}
    }
    \label{tab:target_acc_hatecheck}
    \end{center}
    \tiny Group 1: Females\newline
    Group 2: Gender Transitioned persons\newline
    Group 3: LGBTQ persons\newline
    Group 4: African American persons\newline
    Group 5: People with disabilities\newline
    Group 6: The Muslim community\newline
    Group 7: Immigrants/Migrants
\end{table}
\section{Conclusion}
\label{sec:conclusion}
We have introduced a meta learning technique (MLT) for OOPS detection that combines individual language models to improve performance. 
Further, we have combined MLT with a threshold-moving (TM) technique to further boost the performance of the combined predictor. The experiments with open models confirm that our proposed methodology is 
numerically stable and is able to produce superior results on HS detection as compared to traditional methods. Thus, our technique helps to close the performance gap between non-open and open models~\cite{liang2022holistic} on the tasks considered in this paper. 
Additionally, we also establish theoretical bounds on the combiner weight coefficients to show that MLT behaves well computationally, 
which is supported by the experimental results. 
Overall, we think that our proposal is suited for a wide range of problems spanning various application domains (e.g., other NLP tasks and multimodal classification).


\bibliographystyle{unsrt}  
\bibliography{references}  

\begin{thebibliography}{10}

\bibitem{fortuna2018survey}
Paula Fortuna and S{\'e}rgio Nunes.
\newblock A survey on automatic detection of hate speech in text.
\newblock {\em ACM Computing Surveys (CSUR)}, 51(4):1--30, 2018.

\bibitem{liang2022holistic}
Percy Liang, Rishi Bommasani, Tony Lee, Dimitris Tsipras, Dilara Soylu,
  Michihiro Yasunaga, Yian Zhang, Deepak Narayanan, Yuhuai Wu, Ananya Kumar,
  et~al.
\newblock Holistic evaluation of language models.
\newblock {\em arXiv preprint arXiv:2211.09110}, 2022.

\bibitem{gudibande2023false}
Arnav Gudibande, Eric Wallace, Charlie Snell, Xinyang Geng, Hao Liu, Pieter
  Abbeel, Sergey Levine, and Dawn Song.
\newblock The false promise of imitating proprietary llms, 2023.

\bibitem{graves2014neural}
Alex Graves, Greg Wayne, and Ivo Danihelka.
\newblock Neural turing machines.
\newblock {\em arXiv preprint arXiv:1410.5401}, 2014.

\bibitem{bahdanau2014neural}
Dzmitry Bahdanau, Kyunghyun Cho, and Yoshua Bengio.
\newblock Neural machine translation by jointly learning to align and
  translate.
\newblock {\em arXiv preprint arXiv:1409.0473}, 2014.

\bibitem{luong2015effective}
Minh-Thang Luong, Hieu Pham, and Christopher~D Manning.
\newblock Effective approaches to attention-based neural machine translation.
\newblock {\em arXiv preprint arXiv:1508.04025}, 2015.

\bibitem{vaswani2017attention}
Ashish Vaswani, Noam Shazeer, Niki Parmar, Jakob Uszkoreit, Llion Jones,
  Aidan~N Gomez, {\L}ukasz Kaiser, and Illia Polosukhin.
\newblock Attention is all you need.
\newblock {\em Advances in neural information processing systems}, 30, 2017.

\bibitem{devlin2018bert}
Jacob Devlin, Ming-Wei Chang, Kenton Lee, and Kristina Toutanova.
\newblock Bert: Pre-training of deep bidirectional transformers for language
  understanding.
\newblock {\em arXiv preprint arXiv:1810.04805}, 2018.

\bibitem{yang2019xlnet}
Zhilin Yang, Zihang Dai, Yiming Yang, Jaime Carbonell, Russ~R Salakhutdinov,
  and Quoc~V Le.
\newblock Xlnet: Generalized autoregressive pretraining for language
  understanding.
\newblock {\em Advances in neural information processing systems}, 32, 2019.

\bibitem{brown2020language}
Tom Brown, Benjamin Mann, Nick Ryder, Melanie Subbiah, Jared~D Kaplan, Prafulla
  Dhariwal, Arvind Neelakantan, Pranav Shyam, Girish Sastry, Amanda Askell,
  et~al.
\newblock Language models are few-shot learners.
\newblock {\em Advances in neural information processing systems},
  33:1877--1901, 2020.

\bibitem{rae2021scaling}
Jack~W Rae, Sebastian Borgeaud, Trevor Cai, Katie Millican, Jordan Hoffmann,
  Francis Song, John Aslanides, Sarah Henderson, Roman Ring, Susannah Young,
  et~al.
\newblock Scaling language models: Methods, analysis \& insights from training
  gopher.
\newblock {\em arXiv preprint arXiv:2112.11446}, 2021.

\bibitem{minae2022}
Shervin Minaee, Nal Kalchbrenner, Erik Cambria, Narjes Nikzad, Meysam
  Chenaghlu, and Jianfeng Gao.
\newblock Deep learning--based text classification: A comprehensive review.
\newblock {\em ACM Comput. Surv.}, 54(3), apr 2021.

\bibitem{vidgen2019challenges}
Bertie Vidgen, Alex Harris, Dong Nguyen, Rebekah Tromble, Scott Hale, and Helen
  Margetts.
\newblock Challenges and frontiers in abusive content detection.
\newblock In {\em Proceedings of the Third Workshop on Abusive Language
  Online}. Association for Computational Linguistics, 2019.

\bibitem{waseem2017understanding}
Zeerak Waseem, Thomas Davidson, Dana Warmsley, and Ingmar Weber.
\newblock Understanding abuse: A typology of abusive language detection
  subtasks.
\newblock {\em arXiv preprint arXiv:1705.09899}, 2017.

\bibitem{rossini2020beyond}
Patricia Rossini.
\newblock Beyond incivility: Understanding patterns of uncivil and intolerant
  discourse in online political talk.
\newblock {\em Communication Research}, 49(3):399--425, 2020.

\bibitem{davidson2017automated}
Thomas Davidson, Dana Warmsley, Michael Macy, and Ingmar Weber.
\newblock Automated hate speech detection and the problem of offensive
  language.
\newblock In {\em Proceedings of the international AAAI conference on web and
  social media}, volume~11, pages 512--515, 2017.

\bibitem{poletto2021resources}
Fabio Poletto, Valerio Basile, Manuela Sanguinetti, Cristina Bosco, and Viviana
  Patti.
\newblock Resources and benchmark corpora for hate speech detection: a
  systematic review.
\newblock {\em Language Resources and Evaluation}, 55(2):477--523, 2021.

\bibitem{yin2021towards}
Wenjie Yin and Arkaitz Zubiaga.
\newblock Towards generalisable hate speech detection: a review on obstacles
  and solutions.
\newblock {\em PeerJ Computer Science}, 7:e598, 2021.

\bibitem{davidson2019racial}
Thomas Davidson, Debasmita Bhattacharya, and Ingmar Weber.
\newblock Racial bias in hate speech and abusive language detection datasets.
\newblock {\em arXiv preprint arXiv:1905.12516}, 2019.

\bibitem{waseem2016you}
Zeerak Waseem.
\newblock Are {Y}ou a {R}acist or {A}m {I} {S}eeing {T}hings? {A}nnotator
  {I}nfluence on {H}ate {S}peech {D}etection on {T}witter.
\newblock In {\em Proceedings of the first workshop on NLP and computational
  social science}, pages 138--142, 2016.

\bibitem{golbeck2017large}
Jennifer Golbeck, Zahra Ashktorab, Rashad~O Banjo, Alexandra Berlinger,
  Siddharth Bhagwan, Cody Buntain, Paul Cheakalos, Alicia~A Geller,
  Rajesh~Kumar Gnanasekaran, Raja~Rajan Gunasekaran, et~al.
\newblock A large labeled corpus for online harassment research.
\newblock In {\em Proceedings of the 2017 ACM on web science conference}, pages
  229--233, 2017.

\bibitem{founta2018large}
Antigoni~Maria Founta, Constantinos Djouvas, Despoina Chatzakou, Ilias
  Leontiadis, Jeremy Blackburn, Gianluca Stringhini, Athena Vakali, Michael
  Sirivianos, and Nicolas Kourtellis.
\newblock Large scale crowdsourcing and characterization of twitter abusive
  behavior.
\newblock In {\em Twelfth International AAAI Conference on Web and Social
  Media}, 2018.

\bibitem{vidgen2021introducing}
Bertie Vidgen, Dong Nguyen, Helen Margetts, Patricia Rossini, Rebekah Tromble,
  Kristina Toutanova, Anna Rumshisky, Luke Zettlemoyer, Dilek Hakkani-Tur,
  Iz~Beltagy, et~al.
\newblock Introducing cad: the contextual abuse dataset.
\newblock In {\em Proceedings of the 2021 Conference of the North American
  Chapter of the Association for Computational Linguistics: Human Language
  Technologies}, pages 2289--2303. Association for Computational Linguistics,
  2021.

\bibitem{rottger2020hatecheck}
Paul R{\"o}ttger, Bertram Vidgen, Dong Nguyen, Zeerak Waseem, Helen Margetts,
  and Janet~B Pierrehumbert.
\newblock Hatecheck: Functional tests for hate speech detection models.
\newblock {\em arXiv preprint arXiv:2012.15606}, 2020.

\bibitem{wang2020linformer}
Sinong Wang, Belinda~Z Li, Madian Khabsa, Han Fang, and Hao Ma.
\newblock Linformer: Self-attention with linear complexity.
\newblock {\em arXiv preprint arXiv:2006.04768}, 2020.

\bibitem{GooglePerspective}
{Google Jigsaw Perspective}.
\newblock Perspective api.
\newblock \url{https://perspectiveapi.com/}, 2022.

\bibitem{TwoHat}
{Two Hat}.
\newblock Shift ninja.
\newblock \url{https://www.twohat.com/blog/introducing-sift-ninja/}, 2022.

\bibitem{breiman1996stacked}
Leo Breiman.
\newblock Stacked regressions.
\newblock {\em Machine learning}, 24(1):49--64, 1996.

\bibitem{kittler1998combining}
Josef Kittler, Mohamad Hatef, Robert~PW Duin, and Jiri Matas.
\newblock On combining classifiers.
\newblock {\em IEEE transactions on pattern analysis and machine intelligence},
  20(3):226--239, 1998.

\bibitem{opitz1999popular}
David Opitz and Richard Maclin.
\newblock Popular ensemble methods: An empirical study.
\newblock {\em Journal of artificial intelligence research}, 11:169--198, 1999.

\bibitem{kumar2016ensemble}
Ashnil Kumar, Jinman Kim, David Lyndon, Michael Fulham, and Dagan Feng.
\newblock An ensemble of fine-tuned convolutional neural networks for medical
  image classification.
\newblock {\em IEEE journal of biomedical and health informatics},
  21(1):31--40, 2016.

\bibitem{paul2018predicting}
Rahul Paul, Lawrence Hall, Dmitry Goldgof, Matthew Schabath, and Robert
  Gillies.
\newblock Predicting nodule malignancy using a cnn ensemble approach.
\newblock In {\em 2018 International Joint Conference on Neural Networks
  (IJCNN)}, pages 1--8. IEEE, 2018.

\bibitem{perez2019solo}
F{\'a}bio Perez, Sandra Avila, and Eduardo Valle.
\newblock Solo or ensemble? choosing a cnn architecture for melanoma
  classification.
\newblock In {\em Proceedings of the IEEE/CVF Conference on Computer Vision and
  Pattern Recognition Workshops}, pages 0--0, 2019.

\bibitem{savelli2020multi}
Benedetta Savelli, Alessandro Bria, Mario Molinara, Claudio Marrocco, and
  Francesco Tortorella.
\newblock A multi-context cnn ensemble for small lesion detection.
\newblock {\em Artificial Intelligence in Medicine}, 103:101749, 2020.

\bibitem{cer2018universal}
Daniel Cer, Yinfei Yang, Sheng-yi Kong, Nan Hua, Nicole Limtiaco, Rhomni~St
  John, Noah Constant, Mario Guajardo-Cespedes, Steve Yuan, Chris Tar, et~al.
\newblock Universal sentence encoder.
\newblock {\em arXiv preprint arXiv:1803.11175}, 2018.

\bibitem{ssnet}
Apostol Vassilev, Munawar Hasan, and Honglan Jin.
\newblock Can you tell? ssnet - a biologically-inspired neural network
  framework for sentiment classifiers.
\newblock In Giuseppe Nicosia, Varun Ojha, Emanuele La~Malfa, Gabriele
  La~Malfa, Giorgio Jansen, Panos~M. Pardalos, Giovanni Giuffrida, and Renato
  Umeton, editors, {\em Machine Learning, Optimization, and Data Science},
  pages 357--382, Cham, 2022. Springer International Publishing.

\bibitem{kotsiantis2006handling}
Sotiris Kotsiantis, Dimitris Kanellopoulos, Panayiotis Pintelas, et~al.
\newblock Handling imbalanced datasets: A review.
\newblock {\em GESTS international transactions on computer science and
  engineering}, 30(1):25--36, 2006.

\bibitem{byrd2019effect}
Jonathon Byrd and Zachary Lipton.
\newblock What is the effect of importance weighting in deep learning?
\newblock In {\em International Conference on Machine Learning}, pages
  872--881. PMLR, 2019.

\bibitem{brownlee2020gentle}
Jason Brownlee.
\newblock A gentle introduction to threshold-moving for imbalanced
  classification.
\newblock {\em Machine Learning Mastery}, 2020.

\bibitem{collell2018simple}
Guillem Collell, Drazen Prelec, and Kaustubh~R Patil.
\newblock A simple plug-in bagging ensemble based on threshold-moving for
  classifying binary and multiclass imbalanced data.
\newblock {\em Neurocomputing}, 275:330--340, 2018.

\bibitem{gharroudi2015ensemble}
Ouadie Gharroudi, Haytham Elghazel, and Alex Aussem.
\newblock Ensemble multi-label classification: a comparative study on threshold
  selection and voting methods.
\newblock In {\em 2015 IEEE 27th international conference on tools with
  artificial intelligence (ICTAI)}, pages 377--384. IEEE, 2015.

\bibitem{fan2007study}
Rong-En Fan and Chih-Jen Lin.
\newblock A study on threshold selection for multi-label classification.
\newblock {\em Department of Computer Science, National Taiwan University},
  pages 1--23, 2007.

\bibitem{sun2009classification}
Yanmin Sun, Andrew~KC Wong, and Mohamed~S Kamel.
\newblock Classification of imbalanced data: A review.
\newblock {\em International journal of pattern recognition and artificial
  intelligence}, 23(04):687--719, 2009.

\bibitem{BERTEncoder}
{TensorFlowHub}.
\newblock Bert-cased.
\newblock
  \url{https://tfhub.dev/tensorflow/bert_multi_cased_L-12_H-768_A-12/2}, 2022.

\bibitem{Wolfram:BookCorpus}
{Wolfram Neural Net Repository}.
\newblock Bookcorpus dataset.
\newblock
  \url{https://resources.wolframcloud.com/NeuralNetRepository/resources/BERT-Trained-on-BookCorpus-and-English-Wikipedia-Data},
  2019.

\bibitem{qudar2020tweetbert}
Mohiuddin Md~Abdul Qudar and Vijay Mago.
\newblock Tweetbert: a pretrained language representation model for twitter
  text analysis.
\newblock {\em arXiv preprint arXiv:2010.11091}, 2020.

\bibitem{zaheer2020big}
Manzil Zaheer, Guru Guruganesh, Kumar~Avinava Dubey, Joshua Ainslie, Chris
  Alberti, Santiago Ontanon, Philip Pham, Anirudh Ravula, Qifan Wang, Li~Yang,
  et~al.
\newblock Big bird: Transformers for longer sequences.
\newblock {\em Advances in neural information processing systems},
  33:17283--17297, 2020.

\bibitem{scao2022bloom}
Teven~Le Scao, Angela Fan, Christopher Akiki, Ellie Pavlick, Suzana Ili{\'c},
  Daniel Hesslow, Roman Castagn{\'e}, Alexandra~Sasha Luccioni, Fran{\c{c}}ois
  Yvon, Matthias Gall{\'e}, et~al.
\newblock Bloom: A 176b-parameter open-access multilingual language model.
\newblock {\em arXiv preprint arXiv:2211.05100}, 2022.

\bibitem{dai2019transformer}
Zihang Dai, Zhilin Yang, Yiming Yang, Jaime Carbonell, Quoc~V Le, and Ruslan
  Salakhutdinov.
\newblock Transformer-xl: Attentive language models beyond a fixed-length
  context.
\newblock {\em arXiv preprint arXiv:1901.02860}, 2019.

\bibitem{wolf2019huggingface}
Thomas Wolf, Lysandre Debut, Victor Sanh, Julien Chaumond, Clement Delangue,
  Anthony Moi, Pierric Cistac, Tim Rault, R{\'e}mi Louf, Morgan Funtowicz,
  et~al.
\newblock Huggingface's transformers: State-of-the-art natural language
  processing.
\newblock {\em arXiv preprint arXiv:1910.03771}, 2019.

\bibitem{Trohidis2011-lz}
Konstantinos Trohidis, Grigorios Tsoumakas, George Kalliris, and Ioannis
  Vlahavas.
\newblock Multi-label classification of music by emotion.
\newblock {\em EURASIP Journal on Audio, Speech, and Music Processing},
  2011(1):4, September 2011.

\bibitem{Tsoumakas2010MiningMD}
Grigorios Tsoumakas, Ioannis~Manousos Katakis, and Ioannis~P. Vlahavas.
\newblock Mining multi-label data.
\newblock In {\em Data Mining and Knowledge Discovery Handbook}, 2010.

\bibitem{Zhang2018-lh}
Min-Ling Zhang, Yu-Kun Li, Xu-Ying Liu, and Xin Geng.
\newblock Binary relevance for multi-label learning: an overview.
\newblock {\em Frontiers of Computer Science}, 12(2):191--202, April 2018.

\bibitem{6577846}
Karol Draszawka and Julian Szymański.
\newblock Thresholding strategies for large scale multi-label text classifier.
\newblock In {\em 2013 6th International Conference on Human System
  Interactions (HSI)}, pages 350--355, 2013.

\bibitem{6471714}
Min-Ling Zhang and Zhi-Hua Zhou.
\newblock A review on multi-label learning algorithms.
\newblock {\em IEEE Transactions on Knowledge and Data Engineering},
  26(8):1819--1837, 2014.

\bibitem{tensorflow}
{Google Brain Team}.
\newblock Open source library for ml models.
\newblock \url{https://www.tensorflow.org/}, 2022.

\end{thebibliography}

\clearpage

\section{Appendix: a detailed proof of Theorem~\ref{thm1} }\label{appendix}
\begin{proof}
Here we provide a detailed proof of Theorem~\ref{thm1} in the paper. We note that all equation numbers are local for this appendix, unless explicitly mentioned to be those in the paper. With the notation and assumptions from Section~\ref{sec:combiner}  let
\begin{equation}
  \label{eq:interpolationPred}
  \hat{\mathbf{y}}(\tau) = \frac{1}{|W|}\mathbf{y}(\tau)= \sum_{i=1}^K \frac{w_i}{|W|}\mathbf{y}_i(\tau)
  \end{equation}
be the interpolation predictor constructed as a liner combination of $\mathbf{y}_i$ with coefficients that sum up to $\pm 1$. If the individual predictors are good, then the interpolation predictor $\hat{\mathbf{y}}$ is also good, i.e., $||\mathbf{u} - \sigma_b(\hat{\mathbf{y}})||_t$ is small relative to $||\mathbf{u}||_t$.

Let $\mathbb{L}_t(x)$ be a linear approximation of $\sigma(x)$ for some constant $t>0$, such that  $\mathbb{L}_t(x)$ minimizes $||\mathbb{L}_t(x) - \sigma(x)||_t$. Note that any straight line passing through the points $(\ln(\frac{t}{1-t}),\, t)$ and having the same slope as $\sigma^\prime(\ln(\frac{t}{1-t}))$ satisfies  $||\mathbb{L}_t(x) - \sigma(x)||_t=0$. 

First, consider the case $b< 0$ and $w_i\geq 0$ for $\forall i$ in \eqref{eq:biasedsigma}.

Then, $$||\mathbf{u}-\sigma_b(\mathbf{y})||_t = ||\mathbf{u}-W\mathbf{u} + W\mathbf{u} - \sigma_b(\mathbf{y})||_t = $$
$$||(1-W)\mathbf{u} + W(\mathbf{u} - \frac{1}{W}\sigma_b(\mathbf{y}))||_t.$$

Applying the triangle inequality, we get
\begin{equation}
  ||\mathbf{u}-\sigma_b(\mathbf{y})||_t \geq |1-W|\, ||\mathbf{u}||_t - W ||\mathbf{u} - \frac{1}{W}\sigma_b(\mathbf{y}))||_t.
  \label{eq:triangle}
\end{equation}

Next, consider the case $W \leq 1$. Then, from inequality~\eqref{eq:triangle} 

$$||\mathbf{u}-\sigma_b(\mathbf{y})||_t \geq (1-W) ||\mathbf{u}||_t - W ||\mathbf{u} - \frac{1}{W}\sigma_b(\mathbf{y}))||_t.$$
Let $\mathbb{L}_t$ be as defined above and
\begin{equation}
\label{eq:linearshift}
\mathbb{L}_{t,b}(x) = \mathbb{L}_{t}(x -b).
\end{equation}
Then,
$$||\mathbf{u} - \frac{1}{W}\sigma_b(\mathbf{y}))||_t = ||\mathbf{u} - \frac{1}{W}\mathbb{L}_{t,b}(\mathbf{y}) + \frac{1}{W}\mathbb{L}_{t,b}(\mathbf{y}) -\frac{1}{W}\sigma_b(\mathbf{y}))||_t$$

Thus,
$$||\mathbf{u} - \frac{1}{W}\sigma_b(\mathbf{y}))||_t \leq ||\mathbf{u} - \mathbb{L}_{t,b}(\hat{\mathbf{y}})||_t + \frac{1}{W}||\mathbb{L}_{t,b}(\mathbf{y}) -\sigma_b(\mathbf{y}))||_t.$$
Note that $$\frac{1}{W}||\mathbb{L}_{t,b}(\mathbf{y}) -\sigma_b(\mathbf{y}))||_t=0.$$ From here we get

$$||\mathbf{u}-\sigma_b(\mathbf{y})||_t \geq (1-W) ||\mathbf{u}||_t - W ( ||\mathbf{u} - \mathbb{L}_{t,b}(\hat{\mathbf{y}})||_t).$$

This implies that
$$
  W \geq\frac{ ||\mathbf{u}||_t - ||\mathbf{u}-\sigma_b(\mathbf{y})||_t}{||\mathbf{u}||_t + ||\mathbf{u} - \mathbb{L}_{t,b}(\hat{\mathbf{y}})||_t }.
$$

Note that $||\mathbf{u} - \mathbb{L}_{t,b}(\hat{\mathbf{y}})||_t = ||\mathbf{u} - \sigma_b(\hat{\mathbf{y}})||_t$, because by construction $\mathbf{I}_t(\mathbb{L}_{t,b}(\hat{\mathbf{y}})) = \mathbf{I}_t(\sigma_b(\hat{\mathbf{y}}))$. Hence,
\begin{equation}
\label{eq:leftestimate}
  W \geq\frac{ ||\mathbf{u}||_t - ||\mathbf{u}-\sigma_b(\mathbf{y})||_t}{||\mathbf{u}||_t + ||\mathbf{u} - \sigma_b(\hat{\mathbf{y}})||_t }.
\end{equation}

Note also that 
$||\mathbf{u}-\sigma_b(\mathbf{y})||_t$ is  small, especially with respect to the size of $||\mathbf{u}||_t$. Similarly, by the definition of $\hat{\mathbf{y}}$ in \eqref{eq:interpolationPred}, $||\mathbf{u}-\sigma_b(\hat{\mathbf{y}})||_t$ is small.

Next, consider the case $W > 1$. Then, inequality \eqref{eq:triangle} implies 

\begin{equation}||\mathbf{u}-\sigma_b(\mathbf{y})||_t \geq (W-1) ||\mathbf{u}||_t - W ||\mathbf{u} - \frac{1}{W}\sigma_b(\mathbf{y}))||_t.
\label{eq:triang2}
\end{equation}

Introducing  $\mathbb{L}_{t,b}(x)$ as in the previous case, we get

$$||\mathbf{u} - \frac{1}{W}\sigma_b(\mathbf{y}))||_t = ||\mathbf{u} - \frac{1}{W}\mathbb{L}_{t,b}(\mathbf{y}) + \frac{1}{W}\mathbb{L}_{t,b}(\mathbf{y}) -\frac{1}{W}\sigma_b(\mathbf{y}))||_t$$

Thus,
$$||\mathbf{u} - \frac{1}{W}\sigma_b(\mathbf{y}))||_t \leq ||\mathbf{u} - \mathbb{L}_{t,b}(\hat{\mathbf{y}})||_t + \frac{1}{W}||\mathbb{L}_{t,b}(\mathbf{y}) -\sigma_b(\mathbf{y}))||_t.$$
As we observed above, $W^{-1}||\mathbb{L}_{t,b}(\mathbf{y}) -\sigma_b(\mathbf{y}))||_t=0$ and $||\mathbf{u} - \mathbb{L}_{t,b}(\hat{\mathbf{y}})||_t = ||\mathbf{u} - \sigma_b(\hat{\mathbf{y}})||_t$. Hence,
$$||\mathbf{u} - \frac{1}{W}\sigma_b(\mathbf{y}))||_t \leq ||\mathbf{u} - \sigma_b(\hat{\mathbf{y}})||_t.$$

Substituting this into inequality \eqref{eq:triang2} gives

$$||\mathbf{u}-\sigma_b(\mathbf{y})||_t \geq (W-1) ||\mathbf{u}||_t - W ||\mathbf{u} - \sigma_b(\hat{\mathbf{y}})||_t.$$

From here we get

$$||\mathbf{u}-\sigma_b(\mathbf{y})||_t + ||\mathbf{u}||_t \geq W (||\mathbf{u}||_t -  ||\mathbf{u} - \sigma_b(\hat{\mathbf{y}})||_t).$$

 Note that $||\mathbf{u}-\sigma_b(\mathbf{y})||_t$ and $ ||\mathbf{u} - \sigma_b(\hat{\mathbf{y}})||_t$ are small relative to $||\mathbf{u}||_t$, which means that $||\mathbf{u}||_t -  ||\mathbf{u} - \sigma_b(\hat{\mathbf{y}})||_t >0 $ and not too far from $||\mathbf{u}||_t$.  This implies that 
\begin{equation}
  W \leq \frac{ ||\mathbf{u}||_t + ||\mathbf{u}-\sigma_b(\mathbf{y})||_t}{||\mathbf{u}||_t -  ||\mathbf{u} - \sigma_b(\hat{\mathbf{y}})||_t}.
\label{eq:rightestimate}
\end{equation}

Next, consider the case $b> 0$ and $w_i\leq 0$ for  $\forall i$ in \eqref{eq:biasedsigma}. As before,

$$||\mathbf{u}-\sigma_b(\mathbf{y})||_t = ||\mathbf{u}-W\mathbf{u} + W\mathbf{u} - \sigma_b(\mathbf{y})||_t = $$
$$||(1-W)\mathbf{u} + W(\mathbf{u} - \frac{1}{W}\sigma_b(\mathbf{y}))||_t,$$
and from this we get
\begin{align}
  ||\mathbf{u}-\sigma_b(\mathbf{y})||_t &\geq |1+W|\, ||\mathbf{u}||_t - |W| ||\mathbf{u} + \frac{1}{W}\sigma_b(\mathbf{y})||_t\\
  &= |1+W|\, ||\mathbf{u}||_t -|W| ||\mathbf{u} + \frac{1}{W} \mathbb{L}_{t,b}(\mathbf{y})||_t\\
  &= |1+W|\, ||\mathbf{u}||_t + W ||\mathbf{u} - \frac{1}{|W|} \mathbb{L}_{t,b}(\mathbf{y})||_t\\
  &= |1+W|\, ||\mathbf{u}||_t + W ||\mathbf{u} - \mathbb{L}_{t,b}(\mathbf{\hat y})||_t.
  \label{eq:triangle2}
\end{align}
Furthermore, consider the case $1 + W>0$. Then from \eqref{eq:triangle2} we get
$$
||\mathbf{u}-\sigma_b(\mathbf{y})||_t \geq (1+W)||\mathbf{u}||_t + W ||\mathbf{u} - \mathbb{L}_{t,b}(\mathbf{\hat y})||_t.
$$
Thus, using  $||\mathbf{u} - \mathbb{L}_{t,b}(\hat{\mathbf{y}})||_t = ||\mathbf{u} - \sigma_b(\hat{\mathbf{y}})||_t$, 

\begin{equation}
W\leq -\frac{||\mathbf{u}||_t - ||\mathbf{u}-\sigma_b(\mathbf{y})||_t}{||\mathbf{u}||_t+||\mathbf{u} - \sigma_{b}(\mathbf{\hat y})||_t}
\label{eq:leftnegativeW}
\end{equation}

Now, let $W <-1$. Then,

$$
||\mathbf{u}-\sigma_b(\mathbf{y})||_t \geq -(1+W)||\mathbf{u}||_t + W ||\mathbf{u} - \mathbb{L}_{t,b}(\mathbf{\hat y})||_t.
$$
Form here  using again  $||\mathbf{u} - \mathbb{L}_{t,b}(\hat{\mathbf{y}})||_t = ||\mathbf{u} - \sigma_b(\hat{\mathbf{y}})||_t$, we get
\begin{equation}
W\geq -\frac{||\mathbf{u}||_t + ||\mathbf{u}-\sigma_b(\mathbf{y})||_t}{||\mathbf{u}||_t-||\mathbf{u} - \sigma_{b}(\mathbf{\hat y})||_t}
\label{eq:righttnegativeW}
\end{equation}

Combining  \eqref{eq:leftestimate}, \eqref{eq:rightestimate}, \eqref{eq:leftnegativeW}, and \eqref{eq:righttnegativeW} completes the proof. 
\end{proof}




\end{document}